\documentclass[sigconf]{acmart}

\usepackage{url}            
\usepackage{booktabs}       
\usepackage{amsfonts}       
\usepackage{nicefrac}       
\usepackage{amsthm}
\usepackage{fontawesome}

\usepackage{mathtools}
\usepackage{cleveref}
\usepackage{bm}
\usepackage{tikz}
\usetikzlibrary{matrix,positioning}

\usepackage{wrapfig}
\usepackage{latexsym}

\usepackage{lscape}
\usepackage{algorithm}
\usepackage[noend]{algpseudocode}
\usepackage{subcaption}
\usepackage{enumitem}
\usepackage{soul}
\usepackage{multirow}
\usepackage{multicol}
\usepackage{soul}
\usepackage{dsfont}
\usepackage{adjustbox}
\usepackage{pifont} 
\usepackage{cancel}
\usepackage{tabularx}
\usepackage{longtable}
\usepackage{array}

\usepackage[table]{xcolor}
\definecolor{grad1}{RGB}{255, 255, 204}
\definecolor{grad2}{RGB}{255, 237, 160}
\definecolor{grad3}{RGB}{254, 217, 118}
\definecolor{grad4}{RGB}{254, 178, 76}
\definecolor{grad5}{RGB}{253, 141, 60}

\allowdisplaybreaks[3]

\usepackage{pifont}
\newcommand{\cmark}{\textcolor{green!70!black}{\checkmark}}
\newcommand{\xmark}{\textcolor{red}{\ding{55}}}

\theoremstyle{plain}
\newtheorem{theorem}{Theorem}[section]
\newtheorem{corollary}[theorem]{Corollary}

\newtheorem{remark}[theorem]{Remark}

\newcommand{\std}[1]{\scriptsize{$\pm$#1}}

\AtBeginDocument{%
  }

\copyrightyear{2026}
\acmYear{2026}
\setcopyright{cc}
\setcctype{by}
\acmConference[KDD '26]{Proceedings of the 32nd ACM SIGKDD Conference on Knowledge Discovery and Data Mining V.2}{August 09--13, 2026}{Jeju Island, Republic of Korea}
\acmBooktitle{Proceedings of the 32nd ACM SIGKDD Conference on Knowledge Discovery and Data Mining V.2 (KDD '26), August 09--13, 2026, Jeju Island, Republic of Korea}
\acmDOI{10.1145/3770855.3817898}
\acmISBN{979-8-4007-2259-2/2026/08}

\begin{document}

\title{The Confidence Trap: Calibration Attacks for Graph Neural Networks}

\author{Cuong Dang}
\affiliation{%
  \department{Department of Electrical and Computer Engineering}
  \institution{Virginia Polytechnic Institute and State University}
  \city{Blacksburg}
  \state{VA}
  \country{USA}}
\email{cuongdc@vt.edu}

\author{Jiahao Zhang}
\affiliation{%
  \institution{The Pennsylvania State University}
  \city{University Park}
  \state{PA}
  \country{USA}}
\email{jiahao.zhang@psu.edu}

\author{Hieu Ta Quang}
\affiliation{%
  \department{Center for AI Research}
  \institution{VinUniversity}
  \city{Hanoi}
  \country{Vietnam}}
\email{24hieu.tq@vinuni.edu.vn}

\author{Dung Le}
\affiliation{%
  \department{Center for AI Research}
  \institution{VinUniversity}
  \city{Hanoi}
  \country{Vietnam}}
\email{dung.ld@vinuni.edu.vn}

\author{Lu Cheng}
\affiliation{%
  \institution{University of Illinois at Chicago}
  \city{Chicago}
  \state{IL}
  \country{USA}}
\email{lucheng@uic.edu}

\author{Suhang Wang}
\affiliation{%
  \institution{The Pennsylvania State University}
  \city{University Park}
  \state{PA}
  \country{USA}}
\email{szw494@psu.edu}

\renewcommand{\shortauthors}{Cuong Dang et al.}

\begin{abstract}

While confidence calibration is essential for trustworthy decision-making in safety-critical applications, the robustness of calibrated GNNs to adversarial structural perturbations remains largely unexplored. However, studying calibration attacks on graphs presents unique technical challenges: (1) the discrete nature of graph structures complicates gradient-based optimization, (2) existing underconfidence objectives fail to drive predictions toward uniform distributions, and (3) GNNs are highly sensitive to edge perturbations, often causing unintended label changes that violate attack constraints. To address these challenges, we propose a \textbf{Unified Graph Calibration Attack (UGCA)} framework designed for \textbf{worst-case (white-box) analysis} of GNN calibration robustness. UGCA introduces a KL-divergence loss to encourage uniform predictive distributions, a reranking mechanism to reduce label flipping, a hybrid loss to recover labels when violations occur, and beam search to explore a broader adversarial search space. We further provide theoretical insights linking model generalization, dataset complexity, and calibration vulnerability, showing that models with higher accuracy or trained on datasets with more classes are more susceptible under this threat model. Extensive experiments demonstrate that UGCA substantially increases Expected Calibration Error while preserving classification accuracy. \href{https://github.com/CaptainCuong/Graph-Calibration-Attack.git}{Our code is publicly available at GitHub \faGithub}.

\end{abstract}

\ccsdesc[500]{Computing methodologies~Machine learning}
\ccsdesc[500]{Security and privacy~Web application security}
\keywords{Graph Neural Network, Adversarial Attack, Confidence Calibration}


\maketitle

\section{Introduction}


Graphs are fundamental data structures that naturally represent relationships in various real-world systems, such as social networks~\cite{sankar2021graph,zhang2024graph,zhang2025unlearning,zhang2026attack}, molecular structures~\cite{al2025graph,cuong2023score,wang2022molecular,xu2025dualequinet}, knowledge graphs~\cite{liu2025exposing,yang2026query,luo2026graphs}, scene graphs in machine vision~\cite{zhang2024smartcooper,jiang2025enhancing}, and recommendation engines~\cite{fan2019graph,zhang2024linear,liu2024score}. Graph neural networks (GNNs) provide a powerful way to learn from graph-structured data by leveraging message-passing that iteratively learns a node representation by aggregating its neighborhood information~\cite{kipf2016semi,velivckovic2017graph,zhang2018link}. Although GNNs have achieved significant success, they often suffer from miscalibration \cite{wang2021confident,hsu2022makes,wang2025enhance}, i.e., their confidence scores do not accurately reflect real-world probabilities. This misalignment can lead to critical issues in areas such as financial security, business operations, and customer trust. For example, a hospital might use a GNN to detect cancer from ultrasound scans \cite{chowa2023graph}. If the model underestimates the risk, a patient may feel \emph{falsely} assured and miss the opportunity for early treatment. In contrast, an overconfident false positive could result in unnecessary biopsies or chemotherapy. 

Thus, extensive efforts have been made for GNN calibration, which can be categorized into two categories: (i) post-processing and (ii) in-processing methods. In post-processing, \cite{wang2021confident} introduced a topology-aware post-hoc calibration model that refines confidence estimates while preserving accuracy. In contrast, in-processing methods modify the training process itself; for example, \cite{wang2022gcl} reconfigures the standard cross-entropy loss, assigning greater weight to high-confidence examples and improving overall reliability.




\begin{figure*}[!t]
    \centering
    \includegraphics[width=0.85\textwidth]{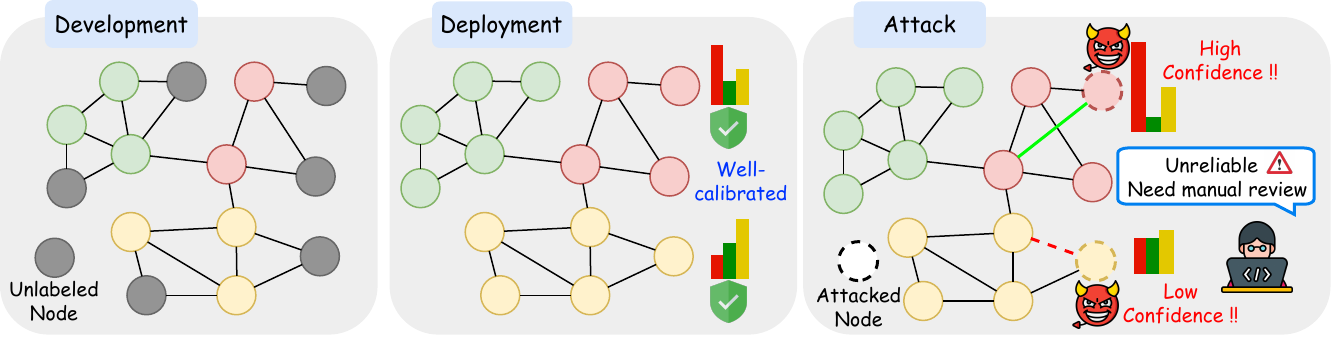}
    \vspace{-1em}
    \caption{\textbf{Illustration of calibration attack scenarios on GNN}. Attackers strategically manipulate graph structure to distort model confidence scores without altering predicted labels. As a result, confidence values become unreliable, either excessively inflated or diminished, necessitating manual intervention and undermining the model's trustworthiness.}
    \label{setting}
\end{figure*}

Despite the widespread use of GNN calibration, its security vulnerabilities remain largely unexplored. In this paper, we address the following research question
: \textit{How robust are calibration methods for GNNs against adversarial attacks?} Our study is conducted within the setting illustrated in Figure \ref{setting}. During the development phase, the GNN undergoes careful training, testing, and calibration to ensure reliable confidence scores. In deployment, predictions typically maintain reasonable confidence levels. However, under adversarial calibration attacks, confidence scores on manipulated nodes become either excessively high or low, while the predicted labels remain unchanged. 
\textit{If accuracy were visibly reduced, developers would quickly detect the problem and retrain or discard the model}. But when accuracy stays high while uncertainty calibration is corrupted, the system appears healthy on paper, so it continues to be deployed in practice. The consequence is subtle but damaging; predictions become less trustworthy, forcing organizations to fall back on costly human reviews instead of relying on the automated system. Thus, calibration attacks not only evade detection but also quietly undermine the very purpose of deploying GNN-based models in the first place. This research question is even more important for safety-critical GNN applications like fraud detection \cite{liu2020alleviating}, medical diagnosis \cite{paul2024systematic}, and autonomous systems \cite{chen2021graph}, where decisions rely heavily on confidence scores. For instance, in fraud detection in financial networks, adversaries can exploit calibration attacks to manipulate the confidence scores of a GNN without changing its classification output. A fraudster may connect to highly trusted nodes in the transaction graph, inheriting strong legitimate embeddings that lead to overconfidence in its classification as “legitimate.” This artificial legitimacy enables the fraudster to evade detection, as the system becomes blind to high-confidence yet incorrect classifications, increasing false negatives. Alternatively, a fraudster could introduce connections to fraudulent and legitimate users, creating ambiguity in classification. As a result, the fraud score for a fraudulent node may drop significantly, for example, from 99\% to 60\%, making it less actionable for automated detection. Although keeping the stealthiness, such manipulations increase manual review costs and delay responses to fraud incidents, weakening the effectiveness of fraud detection mechanisms.



Previous research introduced calibration attacks in the image domain by employing adversarial techniques with under-confidence and over-confidence objectives \cite{obadinma2024calibration,emde2023certified}. These attacks are designed to stop before altering the predicted label, ensuring that the original classification remains intact. However, extending calibration attacks to the graph domain is particularly challenging due to the discrete nature of graph structures, which prevents direct computation of derivatives with respect to the adjacency matrix. As our first contribution, we extend the calibration attack framework to GNNs by incorporating under-confidence and over-confidence objectives within a graph classification attack framework \cite{chen2018fast}. However, as analyzed in Section \ref{analyze}, directly applying prior calibration attack methods \cite{obadinma2024calibration} to GNNs presents several challenges. \textbf{First}, the under-confidence attack objective proves ineffective for graph data. Existing approaches attempt to reduce confidence by minimizing the gap between the highest and second-highest predicted probabilities. However, this method fails to push the probability distribution of a node toward the worst-case scenario, where the output probabilities resemble a uniform distribution. \textbf{Second}, GNNs exhibit high sensitivity to edge perturbations. Even a single modification in the graph structure can cause the second-highest probability to overtake the highest one, violating the attack constraints and rendering the attack unsuccessful. Additionally, edge perturbations may inadvertently alter the target node’s label, disrupting the optimization process. \textbf{Third}, the attack tends to get stuck in local optima. The greedy edge selection strategy, which prioritizes edges that significantly lower confidence, can sometimes lead to unintended label flipping. This premature stopping prevents the attack from reaching a globally optimal solution.

To address these challenges, we propose a Unified Graph Calibration Attack (UGCA) with several key improvements. To enhance the under-confidence objective, we introduce a new formulation based on the KL divergence between the output probability distribution and a uniform distribution. To mitigate the sensitivity of GNNs to edge perturbations, we design a conditional optimization loss function that reduces confidence only when the label remains unchanged and restores the original label when a change occurs. Additionally, we introduce a re-ranking mechanism that refines edge selection to ensure that chosen edges not only decrease confidence but also maintain the predicted label. Finally, to overcome the issue of local optima, we implement beam search, which expands the search space and improves the overall attack performance. These enhancements enable a more effective adaptation of calibration attacks to GNNs, addressing key limitations of the naive approach.

In summary, our \textbf{main contributions} are: (i) The first paper investigates calibration attacks for GNNs, theoretically and empirically analyzing the weaknesses of naive adaptation of the previously proposed calibration attack to GNNs (Section \ref{preliminary_chart}, \ref{preliminary_influence}); (ii) Theoretically presenting the connection between the complexity of the dataset and model generalization and calibration vulnerability (Section \ref{preliminary_theory}); (iii) Formulating graph calibration attacks as a constrained discrete optimization problem and proposing a novel searching algorithm to generate adversarial examples (Section \ref{UGCA}); and (iv) Conducting extensive experiments on five calibration methods across four real-world datasets, providing insights for building trustworthy GNNs (Section \ref{experiment}).

    
    

\section{Related Works}
{\bf Confidence Calibration for Graph Neural Networks.} 
Confidence calibration is a key topic for building trustworthy ML algorithms~\cite{mehrtash2020confidence,ma2023should,ma2024you,nguyen2026urag}. 
Existing calibration methods for GNNs primarily fall into post-processing and in-processing categories. Post-processing methods adjust a model’s confidence scores after training is complete. For example, In-n-Out \cite{nascimento2024n} addresses miscalibration in link prediction tasks by correcting polarity biases in confidence estimation. For node classification, CaGCN \cite{wang2021confident} enhances calibration by propagating confidence scores across connected nodes. Graph Attention Temperature Scaling (GATS) \cite{hsu2022makes} improves calibration through node-specific adjustments guided by graph structural insights. Data-centric methods, notably DCGC \cite{yang2024calibrating}, directly modify adjacency weights to prioritize informative edges, significantly boosting calibration stability. In-processing methods incorporate calibration objectives directly into model training. For example, GCL \cite{wang2022gcl} tackles underconfidence by adding a calibration-aware regularization term to the standard cross-entropy loss. However, these methods assume static, unperturbed graphs, leaving calibration robustness under adversarial manipulation largely unexplored.

\noindent\textbf{Adversarial Attacks for Graph Neural Networks.} 
Adversarial attacks on graph-structured data have become a growing research focus due to the widespread deployment of Graph Neural Networks (GNNs) in sensitive applications. Numerous works have demonstrated that GNNs, like other deep learning models, are vulnerable to subtle perturbations in graph structure or node features \cite{zugner2018adversarial,dai2018adversarial,zugneradversarial,lin2025you}. In addition to graph neural networks, adversaries may attack some other important algorithms for graphs, such as network embeddings including LINE \cite{tang2015line} and Deepwalk \cite{perozzi2014deepwalk}, graph-based semi-supervised learning (G-SSL) \cite{zhu2002learning}, and knowledge graph embedding \cite{lin2015learning}. However, no prior work has explored adversarial attacks in the context of graph calibration \cite{nascimento2024n,wang2022gcl,wang2021confident,hsu2022makes,yang2024calibrating}. This task presents a fundamentally different challenge, rather than altering model predictions, calibration aims to improve the confidence alignment of predictions while keeping labels unchanged. This distinction makes most existing attack strategies unsuitable, as they are designed to intentionally cause label flips. Even minimal perturbations from gradient-based methods often result in label changes, violating the constraints of calibration discussed in Section \ref{preliminary_influence}. To support adversarial studies in graph calibration, there is a need for new attack algorithms that can navigate this delicate balance, perturbing the graph in a way that affects prediction confidence without altering the predicted label. This requires more sophisticated search strategies that can identify and apply subtle, targeted changes to the graph structure, ensuring calibration metrics are influenced while classification outcomes remain stable.

\noindent\textbf{Attacking Uncertainty Estimates.} 
\cite{kopetzki2021evaluating} first showed that predictive uncertainty estimation, specifically Dirichlet-based Uncertainty models, fails to maintain reliable uncertainty estimates when inputs are perturbed. \cite{emde2023certified} directly addressed how adversarial attacks can manipulate not just the accuracy of a model but its uncertainty estimates, linking directly to the broader field of adversarial attacks on uncertainty quantification in machine learning. \cite{kumar2020certifying} introduced a method to certify not just the class label but also the confidence of the classifier's prediction. \cite{kopetzki2021evaluating} is specific to OOD detection, typically assumes a Dirichlet model structure, does not generalize across standard uncertainty estimation techniques, and often impacts prediction accuracy. Hence \cite{galil2021disrupting} presented a general algorithm that disrupts uncertainty estimation across all major techniques, maintaining exact model predictions, making it undetectable by standard robustness checks. Instead of crafting adversarial input in the testing phase, \cite{zeng2023manipulating} tried to attack in the training phase, finding that models show a significant drop in entropy, but in-domain accuracy is still preserved, making it hard to detect. Although many previous works tried to understand the vulnerability of the calibration methods and uncertainty estimation methods in the image domain, they left that of graphs underexplored. However, designing a calibration attack framework for GNN is not obvious; hence, in this paper, we propose a calibration attack specifically for GNNs.
\section{Background}

In this section, we introduce basic knowledge on adversarial graph node classification attacks and calibration attacks. 

\noindent\textbf{ Notations.} 
Given an attributed graph \(\mathcal{G} = (\mathcal{V}, \mathcal{E}, \mathbf{X})\), where \(\mathcal{V} = \{v_1, \dots, v_N\}\) represents the set of \(N\) nodes, \(\mathcal{E} \subseteq \mathcal{V} \times \mathcal{V}\) is the set of edges, and \(\mathbf{X} \in \mathbb{R}^{N \times d}\) is the node feature matrix, where \( d \) is the feature dimension. Each node \( u \) has an initial feature row vector \(\mathbf{x}_u \in \mathbb{R}^{1 \times d}\), and a subset of nodes \( \mathcal{V}_L \subset \mathcal{V} \) have known labels \( y_u \in \{1, \dots, C\} \). The graph’s structure is captured by the adjacency matrix \(\mathbf{A} \in \mathbb{R}^{N \times N}\), where \(\mathbf{A}_{ij} = 1\) if nodes \(v_i\) and \(v_j\) are connected, and \(\mathbf{A}_{ij} = 0\) otherwise. The goal of node classification is to learn a function \( f: (\mathbf{A}, \mathbf{X}) \to \mathbb{R}^{N \times C} \) that maps node features to class probabilities, and the predicted label is denoted by $\hat{y}_u = \mathop{\arg\max}_{k} f(\mathbf{A}, \mathbf{X})_{u,k}$. 

\subsection{Adversarial Node Classification Attack} 

An adversarial attack aims to perturb the graph \( \mathcal{G} \) by modifying a small number of edges, such that the trained model misclassifies the target node. Adversarial modifications can be represented as a perturbation matrix \( \mathbf{M} \): 
\[
\hat{\mathbf{A}} = \mathbf{A} + \mathbf{M},
\]
where \( \mathbf{M} \) is a sparse matrix with values in \( \{1, -1, 0\} \), respectively means an edge \( (v_i, v_j) \) is added, removed or no modification. A widely used gradient-based attack, Fast Gradient Attack (FGA) \cite{chen2018fast}, computes the gradient of the loss function with respect to the adjacency matrix and modifies edges based on the influence of edge to the attack objective. The change of model’s loss $\mathcal{L}$, denoted $\Delta\mathcal{L}$, when modifying an edge \( (v_i, v_j) \) is given by:
   \[
   \Delta \mathcal{L} \approx \frac{\partial \mathcal{L}}{\partial \mathbf{A}_{ij}}\cdot \Delta \mathbf{A}_{ij} = \frac{\partial \mathcal{L}}{\partial \mathbf{A}_{ij}}\cdot \mathbf{M}_{ij},
   \]
where $\Delta\mathbf{A}_{ij}$ is the change of the adjacency matrix when an edge \( (v_i, v_j) \) is flipped. The attacker selects the edge \( (v_{i^*}, v_{j^*}) \) that maximizes the approximated decrease in loss:
   \[
   (v_{i^*}, v_{j^*}) = \mathop{\arg\max}_{(v_{i}, v_{j}) \in \mathcal{E}} \Delta \mathcal{L}.
   \]
Once selected, the adjacency matrix is updated according to the perturbation matrix:
     \[
     \mathbf{A}_{ij} \leftarrow \mathbf{A}_{ij} + \mathbf{M}_{ij}.
     \]
The process is repeated for a fixed number of modifications \( K \).

\subsection{Calibration Attack} 

Calibration attacks are types of adversarial attack that manipulates a machine learning model's confidence scores without changing its predicted labels, leading to severe miscalibration. These attacks threaten the reliability of AI systems, especially in safety-critical applications where confidence scores influence decision-making. \citeauthor{obadinma2024calibration}~\cite{obadinma2024calibration} proposed two types of calibration attacks, \textit{underconfidence} and \textit{overconfidence attack}, for the image domain. In this work, we adapt these concepts to graph-based tasks, especially node classification with GNNs. Let $f(\mathbf{A}, \mathbf{X})_{u,k}$ denote the predicted probability of class $k$ belonging to node $u$, given adjacency matrix $A$ and feature matrix $X$.

\noindent\textbf{Underconfidence Attack (UCA).} This reduces the underconfidence loss, $\mathcal{L}_{\textrm{UCA}}$, on the target node $u$, which is the confidence gap between the top-1 class $k$ and second-best class $j$, making the model overly cautious. Essentially, it optimizes the following problem:
\begin{equation}
    \min_{\mathbf{A}}\mathcal{L}_{\textrm{UCA}}(u, k) = \min_{\mathbf{A}}\{f(\mathbf{A}, \mathbf{X})_{u,k} - \max_{j \neq k} f(\mathbf{A}, \mathbf{X})_{u,j}\}.
    \label{UCA_Loss}
\end{equation}

\noindent\textbf{Overconfidence Attack (OCA).} This reduces the overconfidence loss, $\mathcal{L}_{\textrm{OCA}}$, increasing the confidence score of the predicted class $k$ of node $u$, making the model overly certain. It is formulated as 
\begin{equation}
   \min_{\mathbf{A}}\mathcal{L}_{\textrm{OCA}}(u, k) = \min_{\mathbf{A}}\{1 - f(\mathbf{A}, \mathbf{X})_{u,k}\}.
   \label{OCA_Loss}
\end{equation}

Besides, the attack also need to keep the label unchanged.
\begin{equation}
    \mathop{\arg\max}_{m} f(\mathbf{A}, \mathbf{X})_{u,m} = \mathop{\arg\max}_{n} f(\hat{\mathbf{A}}, \mathbf{X})_{u,n}.
    \label{constraint}
\end{equation}

\section{Preliminary Results \& Theoretical Analyses}
\label{analyze}

In this section, we first present a preliminary study of the technical challenges of graph calibration attacks and then provide a theoretical analysis to motivate our proposed method. 

\begin{figure*}[!t]
    \centering
    \includegraphics[width=0.85\textwidth]{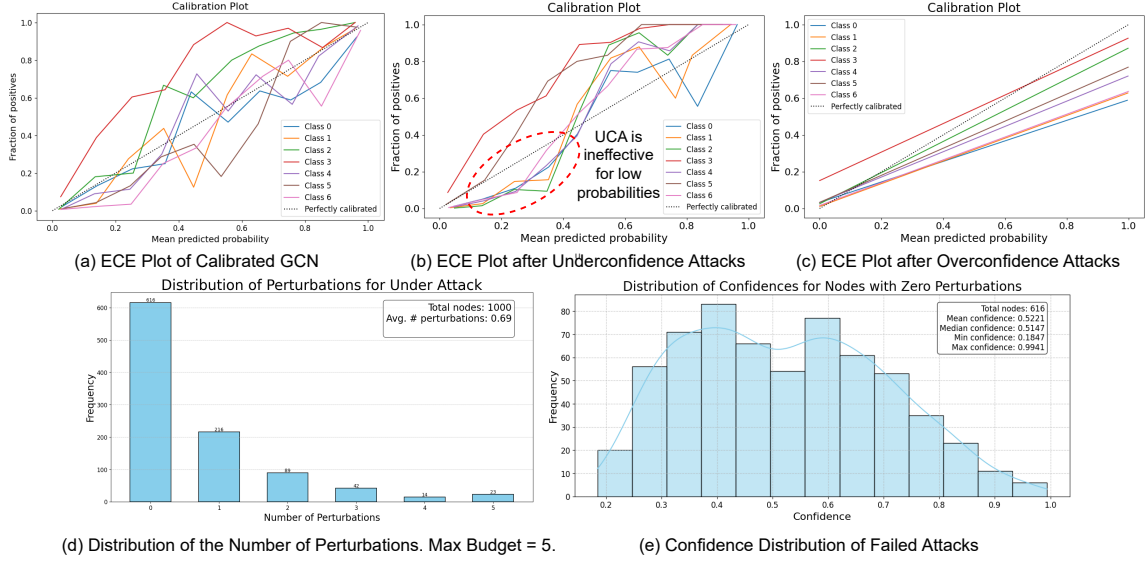}
    \vskip -1em
    \caption{\textbf{Performance of naive UCA and OCA applied to a calibrated GNN using the Fast Gradient Attack (FGA)} \cite{chen2018fast}. (a) Baseline ECE of the calibrated GNN. (b) ECE, after underconfidence attacks, shows limited effectiveness, particularly for low-confidence predictions. (c) ECE after OCA, demonstrating effectiveness due to the assignment of full probability mass to a single class. (d) Distribution of the number of perturbations applied across attack attempts, highlighting a high rate of early termination due to label flipping constraints. (e) Confidence distribution of failed attacks, illustrating the difficulty of reducing confidence for low-probability predictions under the original underconfidence objective.
    }
    \label{preliminary}
\end{figure*}

\subsection{Effectiveness of Under/Overconfidence Attack for GNNs}
\label{preliminary_chart}

To adapt the calibration attack for GNNs, we integrate both overconfidence and underconfidence objectives on top of FGA \cite{chen2018fast} (see Algorithm~\ref{alg:naive} in Appendix~\ref{algo_supp} for details). Figure \ref{preliminary}(c) shows that the overconfidence attack is effective because the attacked confidence is nearly 1, making confidence of the remaining class approximately 0. In contrast, Figure \ref{preliminary}(b) shows that underconfidence attacks struggle with inefficacy, particularly for low-probability predictions. This limitation arises due to the incomplete formulation of the underconfidence objective, inherent weaknesses in the previously proposed underconfidence loss function, and the sensitivity of node probabilities to edge perturbations.

\noindent\textbf{Inherent Limitation of Underconfidence Attack Objective.}
\cite{obadinma2024calibration} suggests that a model’s prediction confidence is lower when the difference between the highest and second-highest probabilities is minimal. However, this interpretation is incomplete; it should be that the confidence of all classes is $\frac{1}{C}$ when the output distribution approximates uniformity.

\noindent\textbf{Weakness of Underconfidence Objective.} \cite{obadinma2024calibration} proposed the underconfidence objective as shown in Equation \ref{UCA_Loss}, making the highest probability decrease while the second highest one increase. Additionally, the rate of change of these probabilities are constant no matter how large they are. This accidentally makes the previous second-highest probability become the highest probability after edge modification, especially when the difference between the highest and second-highest probability is insignificant, prematurely terminating the attack due to a label change.

\noindent\textbf{Label Sensitivity to Edge Perturbation.} As illustrated in Figure \ref{preliminary}(d), a significant proportion of attacks fail with no perturbation, mainly for predictions with confidence around 0.5521 shown in Figure \ref{preliminary}(e). Additionally, even in successful attacks, the modified confidence scores remain distant from the worst-case scenario, where they should ideally converge to $\frac{1}{C}$. This is due to constraints requiring label preservation in Equation \ref{constraint}.

\subsection{Approximate Label Changes by Edge Perturbations}
\label{preliminary_influence}
\label{math_derive}
Whether an edge change makes the label flip is important in selecting an adversarial edge because it helps us to avoid selecting edges likely violating the constraint shown in Equation \ref{constraint}. Next, we mathematically derive the condition on label flipping. Let ($v_{i},v_{j}$) be the perturbed edge, $p_{u,c}$ and $p'_{u,c}$ be the predicted probability of class $c$ of node $u$ before and after the perturbation. The difference between $p'_{u,c}$ and $p_{u,c}$ can be approximated based on the first order approximation as
\[ 
p'_{u,c} - p_{u,c} = \Delta p_{u,c} \approx  \frac{\partial  p_{u,c}}{\partial \mathbf{A}_{ij}} \cdot \Delta \mathbf{A}_{ij}
,\]

Let $m=\mathop{\arg\max}_{k}p_{u,k}, n=\mathop{\arg\max}_{k \neq m}p_{u,k}$, and denote $\hat{p}'_{u,c}$ as the perturbed probability of class $c$ of node $u$. The condition for label flipping is as follows:
\[
\hat{p}'_{u,m} > \hat{p}'_{u,n} \Leftrightarrow \hat{p}_{u,m} + \Delta \hat{p}_{u,m} - \hat{p}_{u,n} - \Delta \hat{p}_{u,n} > 0
\]

\[
\Leftrightarrow \hat{p}_{u,m} - \hat{p}_{u,n} + \frac{\partial  \hat{p}_{u,m}}{\partial \mathbf{A}_{ij}} \cdot \Delta \mathbf{A}_{ij} - \frac{\partial  \hat{p}_{u,n}}{\partial \mathbf{A}_{ij}} \cdot \Delta \mathbf{A}_{ij} > 0 
\]

\[
\Leftrightarrow 
\colorbox{green!20}{$\hat{p}_{u,m} - \hat{p}_{u,n} + \frac{\partial  (\hat{p}_{u,m}-\hat{p}_{u,n})}{\partial A_{ij}} \cdot \Delta \mathbf{A}_{ij} > 0$}.
\]

\subsection{Calibration Vulnerability Through the Lens of Generalization}
\label{preliminary_theory}

\begin{theorem}[Connection between accuracy and attack ECE]
\label{generalization_theorem} Let $\mathcal{D}$ be a test set with $K$ classes. Let $\textrm{acc}(\mathcal{D}) \in [\frac{1}{K},1]$ denote the classification accuracy of a Graph Neural Network model $f$ on $\mathcal{D}$. Ideally, OCA assigns probability 1 to the predicted class and 0 to all others. UCA minimizes the top confidence to $\frac{1}{K}$ without changing the predicted label. Then, the ECE under each attack satisfies:
\[
    \mathrm{ECE}_{\textrm{OCA}} = \frac{1}{K}[(K-2)\textrm{acc}(\mathcal{D})+1], \quad \mathrm{ECE}_{\textrm{UCA}} = \textrm{acc}(\mathcal{D}) - \frac{1}{K}.
\]
\end{theorem}
\begin{proof}
    The proof can be found in Appendix \ref{sec:proof_gen}.
\end{proof} 

In Theorem \ref{generalization_theorem}, we assume that $acc\left({\mathcal{D}}\right)$ lies between $ [\frac{1}{K},1]$. This is a realistic assumption as a random guess gives an accuracy of $\frac{1}{K}$ and a well trained GNN generally performs much better than random guess.

\begin{remark}
     The ECE under both OCA and UCA increases monotonically with classification accuracy $\textrm{acc}(\mathcal{D})$, indicating that better generalization correlates with greater vulnerability to calibration attacks under this model.
\end{remark}

\begin{corollary}[Impact of the number of classes]
\label{generalization_corollary}
Consider two models, $f_1$ and $f_2$, trained on datasets with $K_1$ and $K_2$ classes respectively, where $K_1 > K_2$. Suppose both models achieve identical classification accuracy. Then, under the proposed ideal calibration attacks, the ECE for model $f_1$ exceeds that of model $f_2$; that is, $\mathrm{ECE}_{\textrm{OCA}_1} > \mathrm{ECE}_{\textrm{OCA}_2}$ and $\mathrm{ECE}_{\textrm{UCA}_1} > \mathrm{ECE}_{\textrm{UCA}_2}$.
\end{corollary}
\begin{proof}
    The proof can be found in Appendix \ref{sec:proof_cor}.
\end{proof}

\begin{remark}
    From Theorem~\ref{generalization_theorem} and Corollary~\ref{generalization_corollary}, we can see that as $K \to \infty$, $\mathrm{ECE}_{\text{OCA}}, \mathrm{ECE}_{\text{UCA}}$ strictly increase to $\text{acc}(\mathcal{D})$. This has critical implications. First, high-class tasks (very large $K$) expose an inherent calibration vulnerability that ECE almost asymptotically matches the model accuracy. Besides, as task complexity increases, i.e., $K$ grows, calibration robustness sharply deteriorates, even for accurate models, posing a challenge for applying GNNs to complex classification tasks.
\end{remark}
\section{Unified Graph Calibration Attack Framework}
\label{UGCA}

To address the identified weaknesses in existing approaches (as discussed in Section \ref{analyze}), we formulate the calibration attack on GNNs as a \textbf{constrained discrete optimization problem}. We propose a unified graph calibration attack (UGCA), a novel search algorithm that iteratively optimizes the attack while respecting constraints on label preservation. 

As shown in Figure \ref{framework}, the attack begins with the original graph structure. The algorithm selects an appropriate loss function based on the predicted label of the target node. If the predicted label of the perturbed graph aligns with that of the original one, the attack applies a calibration loss function to manipulate the confidence scores. This includes the overconfidence objective shown in Equation \ref{OCA_Loss} or the underconfidence objective, which is the Kullback-Leibler (KL) divergence between the predicted probabilities and a uniform distribution as follows:

\[
\mathcal{L}_{\mathrm{KL-uniform}} = \sum_{k=1}^C \hat{p}_k \log \frac{\hat{p}_k}{U_k}, \quad U_k = \frac{1}{C}.
\]

\begin{figure}[!t]
    \centering
    \includegraphics[width=0.5\textwidth]{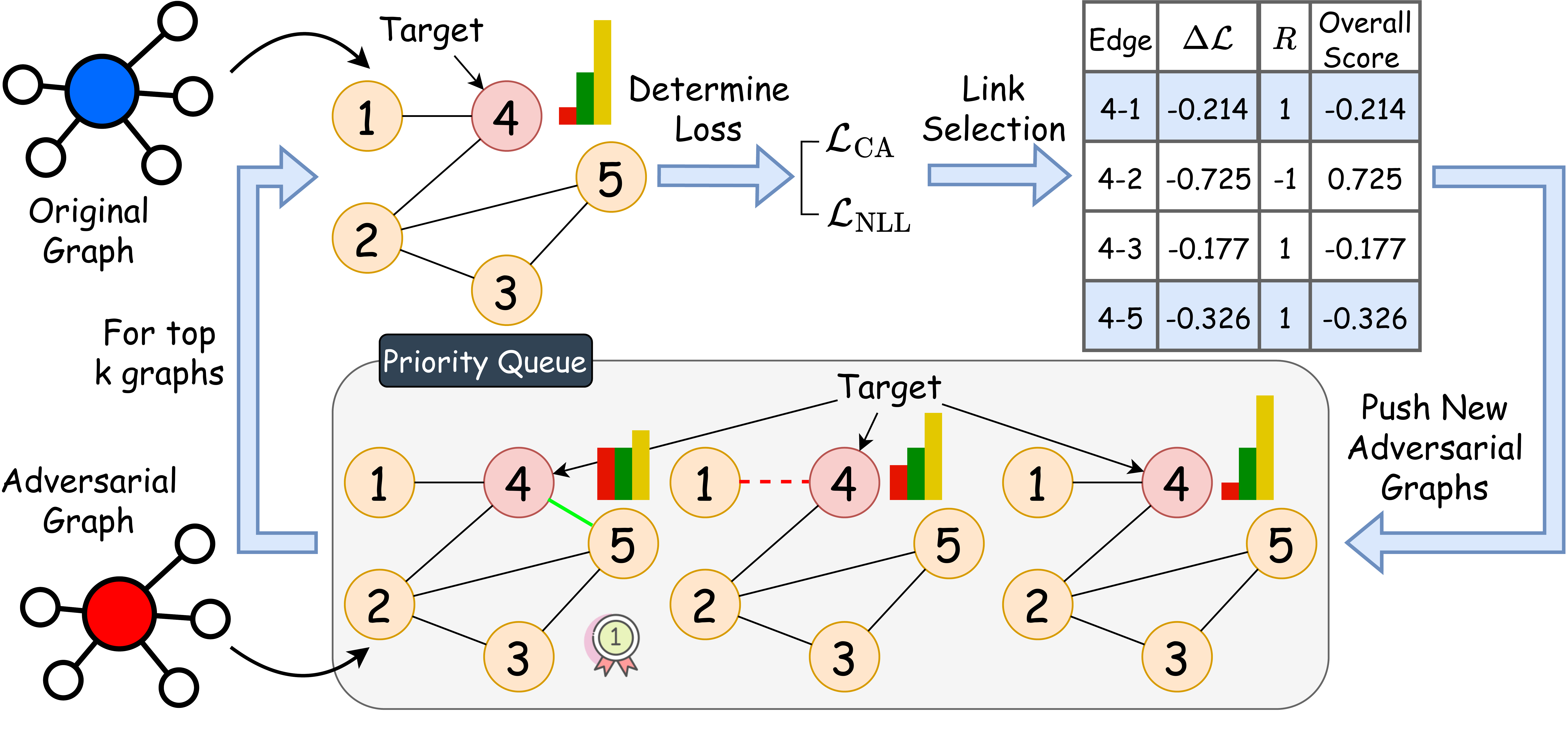}
    \vskip -1em
    \caption{\textbf{Overview of the proposed calibration attack framework}, starting with a benign graph and terminating with an adversarial graph. The process begins by selecting a loss function based on whether the label is preserved, followed by reranking edge perturbations to avoid label flipping. Beam search is then used to explore multiple adversarial graph candidates. This strategy manipulates model confidence while preserving predicted labels, with the most effective adversarial graph selected from the top of the priority queue.}
    \label{framework}
\end{figure}

Compared to the previous underconfidence attack objective proposed in \cite{obadinma2024calibration}, KL-Divergence is more flexible when dealing with low confidence.

\[
\frac{\partial \mathcal{L}_\mathrm{KL-uniform}}{\partial \mathbf{A}} = \sum_{k = 1}^{C} (\log C\hat{p}_{k} + 1)\frac{\partial \hat{p}_{k}}{\partial \mathbf{A}}.
\]

The KL-Divergence derivative shows that during the optimization of the attack objective, a class with low confidence ($\hat{p}_{k}$) yields a small coefficient, $\log C\hat{p}_{k} + 1$. As a result, the gradient $\frac{\partial \hat{p}_{k}}{\partial \mathbf{A}}$ has a reduced effect on minimizing the overall loss, causing only a slight change in the confidence $\partial \hat{p}_{k}$, making label flipping more difficult.

In contrast, when the predicted label is changed, a negative log-likelihood is used for the current probability vector and a one-hot vector with the original predicted label's entry having the value 1.

\[
\mathcal{L}_{\mathrm{NLL}} = -\log p_{u,j}, \quad j = \mathop{\arg\max}_{k} f(\mathbf{A}, \mathbf{X})_{u,k}.
\]

Overall, the loss function can be written in the following format.

\[
\mathcal{L}_{\mathrm{total}} = \beta \cdot \mathcal{L} + (1 - \beta) \cdot \mathcal{L}_{\mathrm{NLL}},
\]

where $\beta = \mathbb{I} (\arg\max f(\mathbf{A}, \mathbf{X}) = y_{u})$, $\mathbb{I}$ is the indicator function, and $y_{u}$ is the groundtruth label of node $u$. This is for flipping back the label when it violates constraint Eq.~\eqref{constraint}. In general, we can consider this as an adaptive attack mechanism, switching between two modes, attacking and validating.

After choosing the loss function, we will take the derivative of the loss with respect to the adjacency matrix to see the influence of edges on the objective. In case of attacking, we need to be aware of whether adding or removing an edge probably violates the constraint. This can be seen as constrained optimization, which optimizes the attack objective but avoids violating the label flipping constraint. To this end, we introduce a \textbf{reranking term} inspired by the derivation in Section \ref{math_derive} to rerank the influence of edges on the objective. The reranking term is defined as:
\[
R = \hat{p}_{u,m} - \hat{p}_{u,n} + \frac{\partial  (\hat{p}_{u,m}-\hat{p}_{u,n})}{\partial \mathbf A_{ij}} \cdot \mathbf A_{ij},
\]
where \( \mathbf A_{ij} \) represents the proposed modification to the adjacency matrix. This term \textbf{downweights} edges that are likely to flip the label while \textbf{amplifying} those that enhance the attack objective without violating constraints.

After calculating the derivative of the attack loss on the adjacency matrix, we multiply it with the reranking term to keep the value of edges not likely to violate the constraint; otherwise, it $R$ helps downweight the opposites. 
\[
S=\Delta\mathcal{L}\cdot R.
\]

For example, as can be seen in Figure \ref{framework}, edge ($v_{4},v_{2}$) is the most influential edge to the attack objective. However, it is likely to change the label; hence, the influence is reversed by multiplying with $R$, leaving edges ($v_{4},v_{1}$) and ($v_{4},v_{5}$) as potential candidates. Then we choose the top $k$, which is the beam size, most influential edges to create $k$ adversarial graphs and push them into a priority queue. We continue this process until it exceeds a budget threshold. 

The details of our proposed unified attack framework are presented in Algorithm~\ref{alg:compare}. We mark the key difference compared with the naive approach in green boxes. For a detailed comparison with the naive algorithm, please refer to Appendix~\ref{algo_supp}.

\begin{algorithm}
\caption{Unified Adversarial Attack Framework for Graph Calibration}
\begin{algorithmic}[1]
\Require Victim Model $f$, Attack Budget $\Delta$, Input Adjacency Matrix $\textbf{A}_{0}$, Attacked Node Index $u$, Groundtruth Label $y_{0}$ of node $u$, Beam Width $B$, Priority Queue $Q$ sorted based on confidence score, Attack Strategy $\mathcal{T}$, Permutation Function $\sigma: \{1, 2, \dots, n\} \to \{1, 2, \dots, n\}$.

\Ensure Adversarial Adjacency Matrix $\textbf{A}_{\text{adv}}$
\State $\hat{p}_0 \gets$ \Call{AttackedConfidence}{$\mathbf{A}_{0},\hat{y}_{0}$}

\State \colorbox{green}{$Q \gets \{\left(\hat{p}_0,\textbf{A}_{0}\right)\}$}

\For{$t = 1$ \textbf{to} $\Delta$}
    \For{\colorbox{green}{$\_,\mathbf{A}$ \textbf{in} $Q.\texttt{top}(B)$}}
        \State $\mathcal{L} \gets$ \Call{DetermineLoss}{$\mathbf{A}$, $y_0$}
        \State $\hat{p} \gets \text{Softmax}\left(f\left(\textbf{A},\mathbf{X}\right)_{u}\right)_{y}$
        \State $V \gets$ \Call{TopEdge}{$\mathcal{L}$, $\mathbf{A}$, $B$, $\hat{p}$} \Comment{Select $B$ potential edges}
        \For{\colorbox{green}{$v$ \textbf{in} $V$}}
            \State $\mathbf{A'}\gets \mathbf{A}$
            \State $\mathbf{A'}_{uv}\gets\mathbf{A'}_{uv}+\mathbf{M}_{uv}$ \Comment{Update adversarial graph}
            \State \colorbox{green}{$Q\gets Q \cup \{($\Call{AttackedConfidence}{$\mathbf{A'},\hat{y}_{0}$}$,\mathbf{A'})\}$} \Comment{Add an adversarial graph}
        \EndFor
    \EndFor
\EndFor
    
\State \colorbox{green}{$\_$, $\mathbf{A}_{\text{adv}}\gets Q.\texttt{top}(1)$}

\State \textbf{return} $\mathbf{A}_{\text{adv}}$

\State

\Procedure{AttackedConfidence}{$\mathbf{A}$, $y$}
\State \textbf{return} $\text{Softmax}\left(f\left(\textbf{A},\mathbf{X}\right)_{u}\right)_{y}$ \Comment{Probability of class $y$ of node $u$}
\EndProcedure

\State

\Procedure{DetermineLoss}{$\mathbf{A}$}
\State $\hat{y} \gets \mathop{\arg\max}_{k} f\left(\textbf{A},\mathbf{X}\right)_{u,k}$
\State $\hat{p} \gets \text{Softmax}\left(f\left(\mathbf{A},\mathbf{X}\right)_{u}\right)$
\If{\colorbox{green}{$\hat{y}\neq y_0$}}
    \State \colorbox{green}{\textbf{return} $-\text{log}$ $\hat{p}_{y_0}$}
\ElsIf{$\mathcal{T}$ \textbf{is} \texttt{OCA}}
    \State \textbf{return} $1-\hat{p}_{\hat{y}}$
\ElsIf{$\mathcal{T}$ \textbf{is} \texttt{UCA}}
    \State \colorbox{green}{$U\gets \text{Uniform}(C)$}
    \State \colorbox{green}{\textbf{return} $\sum_{k=1}^C \hat{p}_k \log \frac{\hat{p}_k}{U_k}$}
\EndIf
\EndProcedure

\State

\Procedure{TopEdge}{$\mathcal{L}$, $\mathbf{A}$, $B$, $p$}
\State $\Delta\mathcal{L} \gets \frac{\partial \mathcal{L}}{\partial \mathbf{A}}$
\State \colorbox{green}{$R \gets p_{u,m} - p_{u,n} + \frac{\partial  \left(p_{u,m}-p_{u,n}\right)}{\partial A}$}
\State \colorbox{green}{$S \gets \Delta \mathcal{L}\cdot R$}
\State $S\gets S_{u}$ \Comment{Just select edges connecting to node $u$}
\State Sort $S$ such that $S_{\sigma(0)}\geq S_{\sigma(1)}\geq\dots\geq S_{\sigma(|S|)}$
\State \textbf{return} \colorbox{green}{$\{\sigma(0),\sigma(1),\cdots,\sigma(B)\}$}
\EndProcedure

\end{algorithmic}
\label{alg:compare}
\end{algorithm}
\section{Experiments}

\label{experiment}

In this section, we present our experimental settings and introduce our empirical results.  

\renewcommand{\std}[1]{\tiny{$\pm$#1}}
\renewcommand{\arraystretch}{0.6}
\newcommand{\gradcell}[2]{\ifnum#1=0\cellcolor{grad1}#2%
\else\ifnum#1=1\cellcolor{grad2}#2%
\else\ifnum#1=2\cellcolor{grad3}#2%
\else\ifnum#1=3\cellcolor{grad4}#2%
\else\ifnum#1=4\cellcolor{grad5}#2%
\else#2\fi\fi\fi\fi\fi}

\subsection{Experimental Settings}

\noindent\textbf{Datasets and Baselines.} We mainly conduct experiments on 4 datasets, including Pubmed \cite{sen2008collective}, Citeseer \cite{sen2008collective}, Cora \cite{sen2008collective}, and CoraML \cite{sen2008collective}, and also include results on Photo~\cite{shchur2018pitfalls}, Physics~\cite{shchur2018pitfalls}, OGBN-Arxiv~\cite{hu2020open}, and Reddit~\cite{hamilton2017inductive}. We focus on a wide range of calibration methods, including TS \cite{guo2017calibration}, VS \cite{guo2017calibration}, MS \cite{guo2017calibration}, SimCalib~\cite{tang2024simcalib}, CaGCN~\cite{wang2021confident}, DCGC~\cite{yang2024calibrating}, ETS~\cite{hsu2022makes}, GATS~\cite{hsu2022makes}, and GETS~\cite{zhuanggets}. 

We follow previous works on adversarial attack \cite{chen2018fast,zugner2018adversarial,dai2018adversarial} and choose GCN \cite{kipf2016semi}, one of the most representative GNN architectures, as the victim model. We use Expected Calibration Error (ECE) as the metric for calibration robustness; a higher ECE means less robust to perturbation. Baseline is the calibration attack proposed in \cite{obadinma2024calibration} built on top of FGA \cite{chen2018fast}. As the underconfidence attack is more technically challenging and less effective in its naive form, our experiments are primarily designed to evaluate the effectiveness and robustness of calibration methods under underconfidence attacks.

Each dataset $\mathcal{D}$ is divided into four disjoint subsets, training set $\mathcal{D}_{\textrm{train}}$, validation set $\mathcal{D}_{\textrm{val}}$, calibration set $\mathcal{D}_{\textrm{cal}}$, test set $\mathcal{D}_{\textrm{test}}$, ensuring that $\mathcal{D}=\bigcup\{\mathcal{D}_{\textrm{train}},\mathcal{D}_{\textrm{val}},\mathcal{D}_{\textrm{cal}},\mathcal{D}_{\textrm{test}}\}$ with no overlap. A GNN is first trained and validated  on the $\mathcal{D}_{\textrm{train}},\mathcal{D}_{\textrm{val}}$ sets. After that, the model is calibrated on $\mathcal{D}_{\textrm{cal}}$ and finally tested on $\mathcal{D}_{\textrm{test}}$. In our experiments, we randomly split such that $|\mathcal{D}_{\textrm{train}}|,|\mathcal{D}_{\textrm{val}}|,|\mathcal{D}_{\textrm{cal}}|,|\mathcal{D}_{\textrm{test}}|$ are $0.6,0.15,0.125,0.125$ of $|\mathcal{D}|$, respectively.

\noindent\textbf{Attacker Knowledge.} The strength of adversarial attacks is maximized in the white-box environment and minimized in the black-box environment. To prioritize the exploration of the most extreme vulnerability of graph calibration methods, we perform the adversarial attack in a white-box setting, where the attackers know the model gradient and the model output confidence scores.

\noindent\textbf{Attacker Capability.} To simulate realistic attack scenarios, such as fraudsters manipulating financial transaction networks, we constrain the attacker's capability to only modify edges directly connected to the target node. In this setting, nodes represent entities (e.g., bank accounts), and edges represent interactions (e.g., transactions). This local perturbation assumption aligns with practical constraints, where adversaries typically have limited access and can only influence their immediate neighborhood. Accordingly, our proposed attack framework and experiments are designed to operate under this restricted setting.

\begin{table}[!ht]
\centering
\small
\caption{Comparison of baselines in underconfidence attack.}
\vspace{-1em}
\begin{tabular}{lccc}
\toprule[1.2pt]
\textbf{Method} & \textbf{Runtime (second)} & \textbf{ASR} & \textbf{Confidence Reduction} \\
\midrule
Random & 2.10  & 100/100 & $-0.237$ \\
IGA~\cite{chen2018link}    & 15921  & 0/100   & $0$      \\
FGA~\cite{chen2018fast}    & 4.91   & 76/100  & $-0.2853$ \\
UCGA (ours)  & 9.54   & 100/100 & $-0.3643$ \\
\bottomrule[1.2pt]
\end{tabular}
\label{tab:baseline_comparison}
\end{table}

\begin{table*}[!ht]
    \centering
    \caption{ECE of a calibrated GNN under underconfidence attack variants, using combinations of KL-Divergence Loss (KL), Reranking (RR), Hybrid Loss (HL), and Beam Search (BS). \cmark = component used; \xmark = not used. Intensified cell colors indicate higher ECE values.}
    \vskip -0.1in
    \begin{adjustbox}{width=0.97\textwidth}
    \begin{tabular}{
    >{\centering\arraybackslash}p{1.2em}
    >{\centering\arraybackslash}p{1.2em}
    >{\centering\arraybackslash}p{1.2em}
    >{\centering\arraybackslash}p{1.2em}
    |l
    ccccccccc}
    \toprule[1.2pt]
    \textbf{KL} & \textbf{RR} & \textbf{HL} & \textbf{BS}
    & \textbf{Dataset}
    & \textbf{CaGCN} & \textbf{DCGC} & \textbf{ETS} & \textbf{GATS} & \textbf{GETS} & \textbf{MS} & \textbf{SimCalib} & \textbf{TS} & \textbf{VS} \\
    \midrule
    \xmark & \xmark & \xmark & \xmark
      & \multirow{5}{*}{CiteSeer}
      & \gradcell{0}{0.0412} & \gradcell{2}{0.0604} & \gradcell{0}{0.0472} & \gradcell{0}{0.0563} & \gradcell{0}{0.0432} & \gradcell{0}{0.0571} & \gradcell{0}{0.0470} & \gradcell{1}{0.0648} & \gradcell{0}{0.0530} \\
    \cmark & \xmark & \xmark & \xmark
      &
      & \gradcell{1}{0.0732} & \gradcell{3}{0.0626} & \gradcell{1}{0.0523} & \gradcell{2}{0.0978} & \gradcell{1}{0.0565} & \gradcell{1}{0.0704} & \gradcell{1}{0.0770} & \gradcell{3}{0.0677} & \gradcell{2}{0.0694} \\
    \cmark & \cmark & \xmark & \xmark
      &
      & \gradcell{2}{0.0826} & \gradcell{4}{\textbf{0.0632}} & \gradcell{3}{0.0660} & \gradcell{2}{0.1024} & \gradcell{2}{0.0597} & \gradcell{2}{0.0718} & \gradcell{3}{0.0933} & \gradcell{2}{0.0655} & \gradcell{3}{0.0737} \\
    \cmark & \cmark & \cmark & \xmark
      &
      & \gradcell{3}{0.0849} & \gradcell{0}{0.0558} & \gradcell{2}{0.0571} & \gradcell{3}{0.1071} & \gradcell{3}{0.0620} & \gradcell{3}{0.0758} & \gradcell{2}{0.0822} & \gradcell{0}{0.0613} & \gradcell{1}{0.0657} \\
    \cmark & \cmark & \cmark & \cmark
      &
      & \gradcell{4}{\textbf{0.1130}} & \gradcell{1}{0.0581} & \gradcell{4}{\textbf{0.0786}} & \gradcell{4}{\textbf{0.1164}} & \gradcell{4}{\textbf{0.0712}} & \gradcell{4}{\textbf{0.0858}} & \gradcell{4}{\textbf{0.1125}} & \gradcell{4}{\textbf{0.1000}} & \gradcell{4}{\textbf{0.0816}} \\
    \midrule
    \xmark & \xmark & \xmark & \xmark
      & \multirow{5}{*}{Cora}
      & \gradcell{0}{0.0524} & \gradcell{0}{0.0277} & \gradcell{0}{0.0473} & \gradcell{0}{0.0539} & \gradcell{0}{0.0336} & \gradcell{0}{0.0441} & \gradcell{0}{0.0554} & \gradcell{0}{0.0478} & \gradcell{0}{0.0511} \\
    \cmark & \xmark & \xmark & \xmark
      &
      & \gradcell{1}{0.0743} & \gradcell{1}{0.0282} & \gradcell{1}{0.0708} & \gradcell{1}{0.1037} & \gradcell{1}{0.0465} & \gradcell{1}{0.0473} & \gradcell{1}{0.1018} & \gradcell{2}{0.0726} & \gradcell{1}{0.0627} \\
    \cmark & \cmark & \xmark & \xmark
      &
      & \gradcell{3}{0.0894} & \gradcell{3}{0.0292} & \gradcell{2}{0.0763} & \gradcell{2}{0.1083} & \gradcell{2}{0.0546} & \gradcell{3}{0.0540} & \gradcell{2}{0.1080} & \gradcell{3}{0.0794} & \gradcell{3}{0.0727} \\
    \cmark & \cmark & \cmark & \xmark
      &
      & \gradcell{2}{0.0772} & \gradcell{4}{\textbf{0.0293}} & \gradcell{3}{0.0772} & \gradcell{3}{0.1073} & \gradcell{3}{0.0616} & \gradcell{2}{0.0524} & \gradcell{3}{0.1095} & \gradcell{1}{0.0714} & \gradcell{2}{0.0701} \\
    \cmark & \cmark & \cmark & \cmark
      &
      & \gradcell{4}{\textbf{0.1081}} & \gradcell{2}{0.0284} & \gradcell{4}{\textbf{0.0964}} & \gradcell{4}{\textbf{0.1248}} & \gradcell{4}{\textbf{0.0657}} & \gradcell{4}{\textbf{0.0725}} & \gradcell{4}{\textbf{0.1241}} & \gradcell{4}{\textbf{0.0890}} & \gradcell{4}{\textbf{0.0909}} \\
    \midrule
    \xmark & \xmark & \xmark & \xmark
      & \multirow{5}{*}{CoraML}
      & \gradcell{0}{0.0413} & \gradcell{2}{0.0279} & \gradcell{0}{0.0413} & \gradcell{0}{0.0467} & \gradcell{1}{0.0395} & \gradcell{1}{0.0394} & \gradcell{0}{0.0438} & \gradcell{0}{0.0408} & \gradcell{1}{0.0423} \\
    \cmark & \xmark & \xmark & \xmark
      &
      & \gradcell{1}{0.0601} & \gradcell{0}{0.0227} & \gradcell{1}{0.0450} & \gradcell{1}{0.0735} & \gradcell{2}{0.0435} & \gradcell{3}{0.0420} & \gradcell{1}{0.0572} & \gradcell{1}{0.0441} & \gradcell{0}{0.0390} \\
    \cmark & \cmark & \xmark & \xmark
      &
      & \gradcell{3}{0.0774} & \gradcell{1}{0.0250} & \gradcell{3}{0.0505} & \gradcell{3}{0.0923} & \gradcell{3}{0.0455} & \gradcell{0}{0.0380} & \gradcell{3}{0.0662} & \gradcell{3}{0.0528} & \gradcell{2}{0.0458} \\
    \cmark & \cmark & \cmark & \xmark
      &
      & \gradcell{2}{0.0677} & \gradcell{3}{0.0282} & \gradcell{2}{0.0464} & \gradcell{2}{0.0857} & \gradcell{0}{0.0355} & \gradcell{2}{0.0405} & \gradcell{2}{0.0578} & \gradcell{2}{0.0503} & \gradcell{3}{0.0497} \\
    \cmark & \cmark & \cmark & \cmark
      &
      & \gradcell{4}{\textbf{0.0947}} & \gradcell{4}{\textbf{0.0302}} & \gradcell{4}{\textbf{0.0706}} & \gradcell{4}{\textbf{0.1123}} & \gradcell{4}{\textbf{0.0558}} & \gradcell{4}{\textbf{0.0485}} & \gradcell{4}{\textbf{0.0771}} & \gradcell{4}{\textbf{0.0686}} & \gradcell{4}{\textbf{0.0588}} \\
    \midrule
    \xmark & \xmark & \xmark & \xmark
      & \multirow{5}{*}{Ogbn-arxiv}
      & \gradcell{0}{0.0084} & \gradcell{0}{0.0099} & \gradcell{0}{0.0132} & \gradcell{0}{0.0091} & \gradcell{0}{0.0076} & \gradcell{0}{0.0219} & \gradcell{1}{0.0189} & \gradcell{0}{0.01} & \gradcell{0}{0.0071} \\
    \cmark & \xmark & \xmark & \xmark
      &
      & \gradcell{1}{0.0124} & \gradcell{1}{0.0100} & \gradcell{1}{0.0136} & \gradcell{1}{0.0115} & \gradcell{1}{0.0112} & \gradcell{2}{0.0246} & \gradcell{0}{0.0172} & \gradcell{1}{0.0111} & \gradcell{2}{0.0093} \\
    \cmark & \cmark & \xmark & \xmark
      &
      & \gradcell{3}{0.0152} & \gradcell{2}{0.0100} & \gradcell{2}{0.0146} & \gradcell{2}{0.0115} & \gradcell{2}{0.0129} & \gradcell{1}{0.0244} & \gradcell{3}{0.0219} & \gradcell{3}{0.0133} & \gradcell{3}{0.0108} \\
    \cmark & \cmark & \cmark & \xmark
      &
      & \gradcell{2}{0.0136} & \gradcell{3}{0.0101} & \gradcell{3}{0.0166} & \gradcell{3}{0.0123} & \gradcell{2}{0.0124} & \gradcell{3}{0.0247} & \gradcell{2}{0.0216} & \gradcell{2}{0.0122} & \gradcell{1}{0.0088} \\
    \cmark & \cmark & \cmark & \cmark
      &
      & \gradcell{4}{\textbf{0.0210}} & \gradcell{4}{\textbf{0.0109}} & \gradcell{4}{\textbf{0.0194}} & \gradcell{4}{\textbf{0.0179}} & \gradcell{4}{\textbf{0.0179}} & \gradcell{4}{\textbf{0.0274}} & \gradcell{4}{\textbf{0.0230}} & \gradcell{4}{\textbf{0.0163}} & \gradcell{4}{\textbf{0.0169}} \\
    \midrule
    \xmark & \xmark & \xmark & \xmark
      & \multirow{5}{*}{Photo}
      & \gradcell{0}{0.0247} & \gradcell{1}{0.0085} & \gradcell{0}{0.0294} & \gradcell{0}{0.0296} & \gradcell{0}{0.0205} & \gradcell{0}{0.0326} & \gradcell{0}{0.0374} & \gradcell{0}{0.0247} & \gradcell{0}{0.0238} \\
    \cmark & \xmark & \xmark & \xmark
      &
      & \gradcell{2}{0.0314} & \gradcell{3}{0.0091} & \gradcell{1}{0.0378} & \gradcell{4}{\textbf{0.0425}} & \gradcell{1}{0.0255} & \gradcell{1}{0.0365} & \gradcell{1}{0.0403} & \gradcell{2}{0.0299} & \gradcell{1}{0.0272} \\
    \cmark & \cmark & \xmark & \xmark
      &
      & \gradcell{1}{0.0295} & \gradcell{2}{0.0087} & \gradcell{2}{0.0397} & \gradcell{3}{0.0402} & \gradcell{4}{\textbf{0.0322}} & \gradcell{2}{0.0372} & \gradcell{2}{0.0479} & \gradcell{3}{0.0301} & \gradcell{2}{0.0288} \\
    \cmark & \cmark & \cmark & \xmark
      &
      & \gradcell{3}{0.0343} & \gradcell{0}{0.0074} & \gradcell{3}{0.0408} & \gradcell{1}{0.0365} & \gradcell{3}{0.0299} & \gradcell{3}{0.0376} & \gradcell{3}{0.0507} & \gradcell{4}{\textbf{0.0338}} & \gradcell{3}{0.0300} \\
    \cmark & \cmark & \cmark & \cmark
      &
      & \gradcell{4}{\textbf{0.0448}} & \gradcell{4}{\textbf{0.0098}} & \gradcell{4}{\textbf{0.0470}} & \gradcell{2}{0.0375} & \gradcell{2}{0.0263} & \gradcell{4}{\textbf{0.0693}} & \gradcell{4}{\textbf{0.0573}} & \gradcell{1}{0.0298} & \gradcell{4}{\textbf{0.0301}} \\
    \midrule
    \xmark & \xmark & \xmark & \xmark
      & \multirow{5}{*}{Physics}
      & \gradcell{0}{0.0669} & \gradcell{0}{0.0053} & \gradcell{0}{0.0761} & \gradcell{0}{0.0838} & \gradcell{1}{0.0489} & \gradcell{0}{0.0361} & \gradcell{0}{0.0748} & \gradcell{0}{0.0682} & \gradcell{0}{0.0692} \\
    \cmark & \xmark & \xmark & \xmark
      &
      & \gradcell{1}{0.0733} & \gradcell{1}{0.0079} & \gradcell{1}{0.0756} & \gradcell{2}{0.1063} & \gradcell{0}{0.0447} & \gradcell{1}{0.0429} & \gradcell{2}{0.0996} & \gradcell{1}{0.0740} & \gradcell{1}{0.0722} \\
    \cmark & \cmark & \xmark & \xmark
      &
      & \gradcell{2}{0.0906} & \gradcell{2}{0.0085} & \gradcell{2}{0.1001} & \gradcell{1}{0.1033} & \gradcell{2}{0.0554} & \gradcell{3}{0.0575} & \gradcell{1}{0.0941} & \gradcell{2}{0.0771} & \gradcell{3}{0.0919} \\
    \cmark & \cmark & \cmark & \xmark
      &
      & \gradcell{3}{0.1020} & \gradcell{3}{0.0095} & \gradcell{2}{0.0932} & \gradcell{3}{0.1093} & \gradcell{3}{0.0672} & \gradcell{2}{0.0536} & \gradcell{3}{0.1034} & \gradcell{3}{0.0821} & \gradcell{2}{0.0762} \\
    \cmark & \cmark & \cmark & \cmark
      &
      & \gradcell{4}{\textbf{0.1307}} & \gradcell{4}{\textbf{0.0109}} & \gradcell{4}{\textbf{0.1183}} & \gradcell{4}{\textbf{0.1200}} & \gradcell{4}{\textbf{0.0778}} & \gradcell{4}{\textbf{0.0814}} & \gradcell{4}{\textbf{0.1200}} & \gradcell{4}{\textbf{0.1125}} & \gradcell{4}{\textbf{0.1029}} \\
    \midrule
    \xmark & \xmark & \xmark & \xmark
      & \multirow{5}{*}{PubMed}
      & \gradcell{0}{0.0521} & \gradcell{1}{0.0618} & \gradcell{0}{0.0515} & \gradcell{0}{0.0646} & \gradcell{0}{0.0540} & \gradcell{1}{0.0763} & \gradcell{0}{0.0516} &  \gradcell{1}{0.0743} & \gradcell{2}{0.0809} \\
    \cmark & \xmark & \xmark & \xmark
      &
      & \gradcell{3}{0.0846} & \gradcell{0}{0.0617} & \gradcell{1}{0.0625} & \gradcell{2}{0.0843} & \gradcell{1}{0.0646} & \gradcell{2}{0.0780} & \gradcell{3}{0.0798} &  \gradcell{0}{0.0741} & \gradcell{1}{0.0779} \\
    \cmark & \cmark & \xmark & \xmark
      &
      & \gradcell{2}{0.0797} & \gradcell{2}{0.0631} & \gradcell{2}{0.0663} & \gradcell{1}{0.0806} & \gradcell{3}{0.0801} & \gradcell{0}{0.0762} & \gradcell{2}{0.0776} & \gradcell{2}{0.0759} & \gradcell{3}{0.0843} \\
    \cmark & \cmark & \cmark & \xmark
      &
      & \gradcell{1}{0.0748} & \gradcell{3}{0.0646} & \gradcell{3}{0.0729} & \gradcell{3}{0.0861} & \gradcell{2}{0.0684} & \gradcell{3}{0.0786} & \gradcell{1}{0.0764} & \gradcell{3}{0.0849} & \gradcell{0}{0.0766} \\
    \cmark & \cmark & \cmark & \cmark
      &
      & \gradcell{4}{\textbf{0.1521}} & \gradcell{4}{\textbf{0.0654}} & \gradcell{4}{\textbf{0.1622}} & \gradcell{4}{\textbf{0.1508}} & \gradcell{4}{\textbf{0.1199}} & \gradcell{4}{\textbf{0.1351}} & \gradcell{4}{\textbf{0.1516}} & \gradcell{4}{\textbf{0.1347}} & \gradcell{4}{\textbf{0.1449}} \\
    \midrule
    \xmark & \xmark & \xmark & \xmark
      & \multirow{5}{*}{Reddit}
      & \gradcell{0}{--} & \gradcell{0}{0.0019} & \gradcell{1}{0.0088} & \gradcell{0}{0.0083} & \gradcell{1}{0.0044} & \gradcell{0}{0.0058} & \gradcell{0}{0.0084} & \gradcell{0}{0.0039} & \gradcell{0}{0.0042} \\
    \cmark & \xmark & \xmark & \xmark
      &
      & \gradcell{0}{--} & \gradcell{2}{0.0022} & \gradcell{0}{0.0069} & \gradcell{1}{0.0086} & \gradcell{0}{0.0042} & \gradcell{1}{0.0061} & \gradcell{1}{0.0087} & \gradcell{1}{0.0051} & \gradcell{1}{0.0047} \\
    \cmark & \cmark & \xmark & \xmark
      &
      & \gradcell{0}{--} & \gradcell{2}{0.0023} & \gradcell{3}{0.0097} & \gradcell{4}{\textbf{0.0103}} & \gradcell{2}{0.0054} & \gradcell{2}{0.0091} & \gradcell{3}{0.0114} & \gradcell{2}{0.0051} & \gradcell{3}{0.0056} \\
    \cmark & \cmark & \cmark & \xmark
      &
      & \gradcell{0}{--} & \gradcell{3}{0.0023} & \gradcell{2}{0.0094} & \gradcell{2}{0.0088} & \gradcell{4}{\textbf{0.0060}} & \gradcell{3}{0.0091} & \gradcell{2}{0.0098} & \gradcell{3}{0.0051} & \gradcell{2}{0.0052} \\
    \cmark & \cmark & \cmark & \cmark
      &
      & \gradcell{0}{--} & \gradcell{4}{\textbf{0.0023}} & \gradcell{4}{\textbf{0.0109}} & \gradcell{3}{0.0098} & \gradcell{3}{0.0057} & \gradcell{4}{\textbf{0.0128}} & \gradcell{4}{\textbf{0.0115}} & \gradcell{4}{\textbf{0.0058}} & \gradcell{4}{\textbf{0.0068}} \\
    \bottomrule[1.2pt]
    \end{tabular}
    \end{adjustbox}
    \label{tab:ugca_underconfidence_ece_multi}
\end{table*}

\begin{figure}[!t]
    \centering

    \includegraphics[width=\columnwidth]{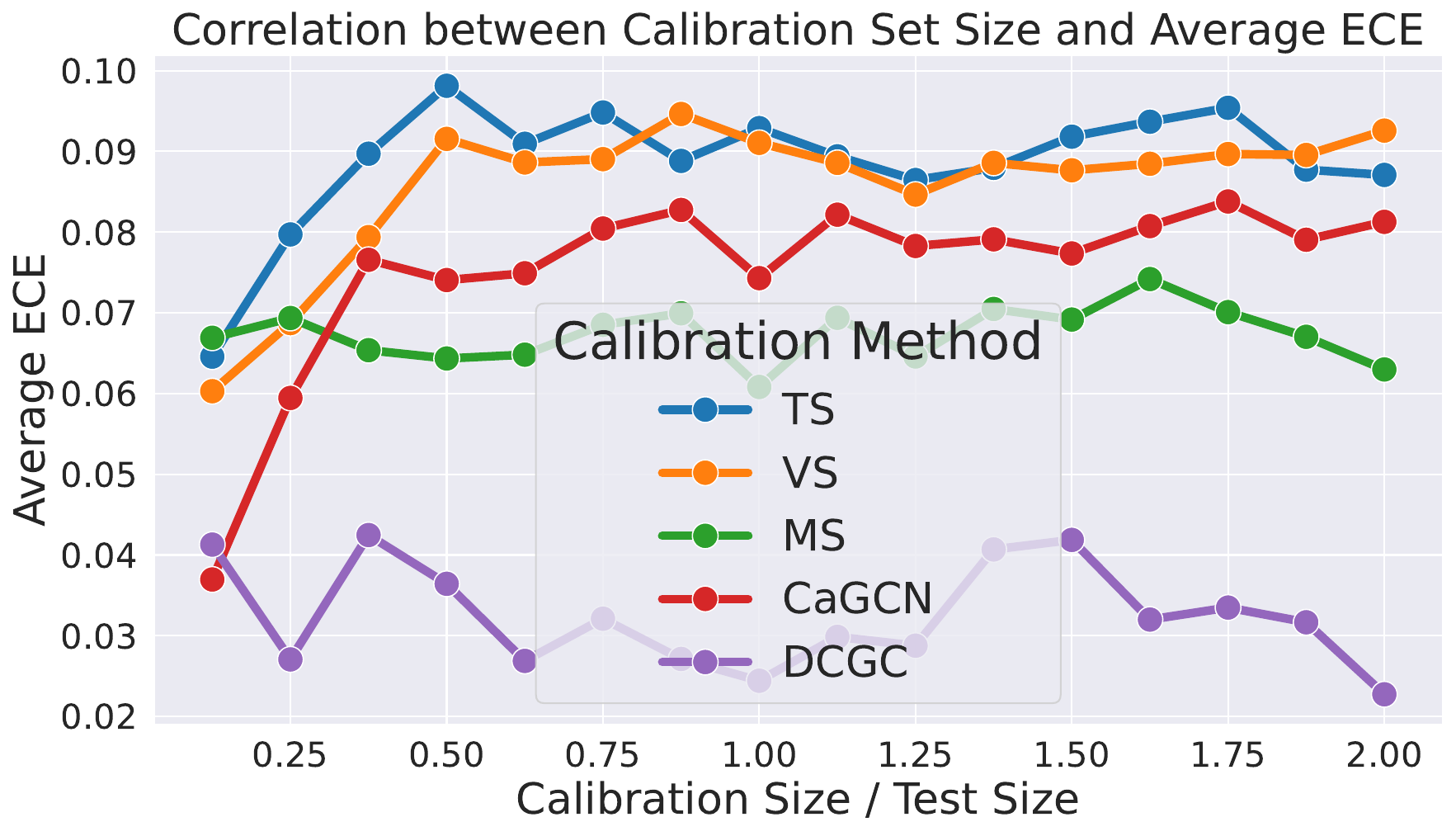}
    \vskip -1em
    \caption{ECE of an attacked GNN calibrated by different calibration methods on the CORA dataset as the calibration set size increases relative to the test set size.}
    \label{fig:calib_size}

    \small
    \vskip 1em
    \captionof{table}{ECE under overconfidence attack. Subscripts indicate ECE of the well-calibrated clean (non-attacked) GNN.}
    \label{tab:ece_oca}
    \vskip -1em
    \resizebox{\columnwidth}{!}{
    \begin{tabular}{lccccc}
    \toprule[1.2pt]
    \textbf{Dataset} & \textbf{TS} & \textbf{VS} & \textbf{MS} & \textbf{CaGCN} & \textbf{DCGC} \\
    \midrule
    Cora     & 0.070{\tiny 0.045} & 0.057{\tiny 0.029} & 0.066{\tiny 0.052} & 0.053{\tiny 0.019} & 0.041{\tiny 0.047} \\
    CoraML   & 0.033{\tiny 0.016} & 0.031{\tiny 0.016} & 0.040{\tiny 0.023} & 0.033{\tiny 0.019} & 0.018{\tiny 0.018} \\
    PubMed   & 0.0617{\tiny 0.055} & 0.075{\tiny 0.042} & 0.042{\tiny 0.010} & 0.088{\tiny 0.130} & 0.045{\tiny 0.055} \\
    Citeseer & 0.041{\tiny 0.047} & 0.064{\tiny 0.040} & 0.031{\tiny 0.032} & 0.061{\tiny 0.034} & 0.039{\tiny 0.054} \\
    \bottomrule[1.2pt]
    \end{tabular}
    }
\end{figure}

\begin{figure*}[!ht]
    \centering

    \begin{subfigure}{0.974\textwidth}
        \centering
        \includegraphics[width=\textwidth]{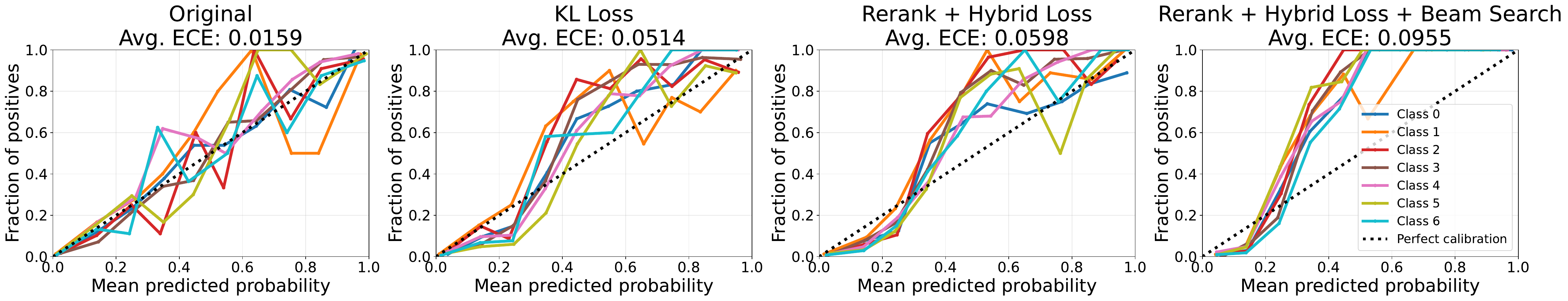}
        \caption{The ECE lines consistently increase, indicating stronger underconfidence, as more components are added to the underconfidence attack. This demonstrates that each component contributes incrementally to degrading calibration performance.}
        \label{fig:methods_compare_plot}
    \end{subfigure}


    \begin{subfigure}{0.975\textwidth}
        \centering
        \includegraphics[width=\textwidth]{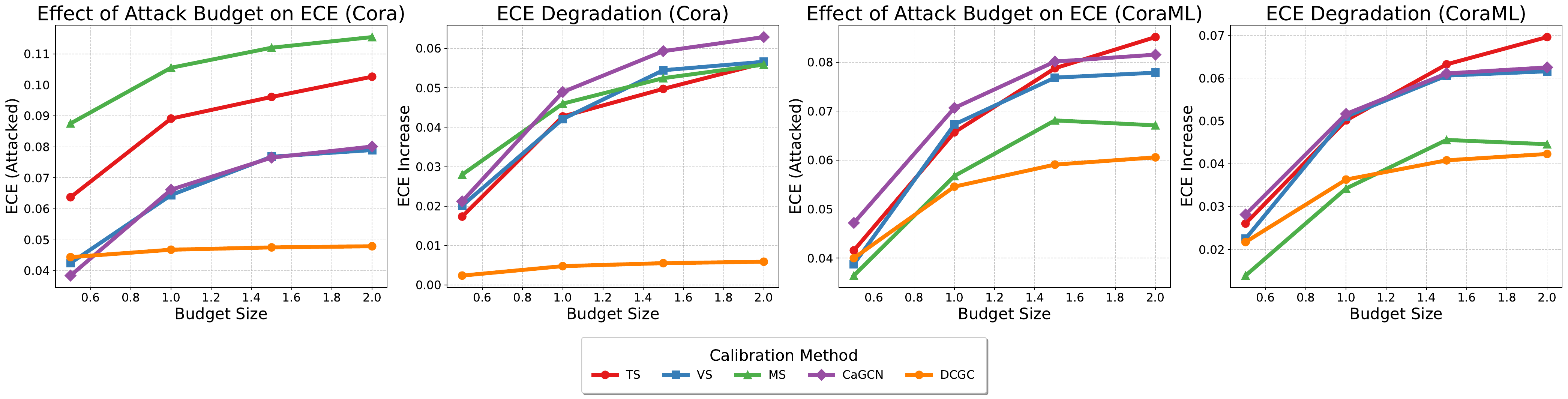}
        \caption{ECE and ECE Increase under adversarial budget scaling for calibration methods on Cora and CoraML. The attack budget is scaled relative to the average node degree of the graph. ECE Increase measures calibration degradation from the clean, unattacked state.}
        \label{fig:budget_chart}
    \end{subfigure}
    \vspace{-0.1in}
    \caption{Calibration robustness analysis under underconfidence attacks and adversarial budget scaling.}
    \label{fig:wrapped_two_figures}
\end{figure*}

\subsection{Baseline Comparison}

We compare our proposed method with a Random baseline that uniformly flips edges connected to the target node. In addition, we adapt two representative gradient-based graph attack methods, IGA \cite{chen2018link} and FGA \cite{chen2018fast}, by directly optimizing the original calibration attack objectives defined in Equations \ref{UCA_Loss} and \ref{OCA_Loss}, while enforcing the label-preservation constraint.

Table \ref{tab:baseline_comparison} shows a clear trade-off between efficiency and effectiveness. The Random baseline is fast and achieves a nominal 100\% ASR, but yields only modest confidence reduction. IGA is prohibitively expensive in runtime and completely fails (0\% ASR). FGA offers a good balance, achieving substantial confidence reduction with moderate runtime, but still fails in a non-trivial fraction of cases. In contrast, UCGA (ours) attains the largest confidence reduction while maintaining 100\% ASR and only a modest increase in runtime compared to FGA, demonstrating that the proposed method is effective without sacrificing practicality. In Appendix~\ref{complexity}, we show that when dealing with large-scale graphs, the efficiency of UGCA is approximately equal to that of the baseline.

\subsection{Ablation of Calibration Attack Components}

To evaluate the individual contributions of each component in our proposed attack framework, we performed an ablation study on the following modules: KL-Divergence Loss (KL), Reranking (RR), Hybrid Loss (HL), and Beam Search (BS). Table \ref{tab:ugca_underconfidence_ece_multi} reveals that integrating multiple attack components systematically enhances the effectiveness of calibration attacks by significantly increasing the ECE. Each component incrementally contributes to miscalibration. Incorporating KL alone yields a moderate improvement in miscalibration, establishing its fundamental role. Subsequently, the inclusion of RR further amplifies this effect (e.g., Cora TS: from 0.0478 to 0.0890) by refining edge perturbation selection. Adding HL generally continues this trend. The addition of BS significantly magnifies the adversarial impact (e.g., PubMed TS: ECE rises sharply from 0.0743 to 0.1347) by extensively broadening the adversarial search space. Figure \ref{fig:methods_compare_plot} further highlights the cumulative effect of components in UCA, illustrating that the accuracy-confidence graphs across all classes are consistently uplifted as more attack components are introduced. Regarding OCA, Table \ref{tab:ece_oca} shows the consistent effectiveness across 4 datasets and 5 calibration methods.

\subsection{Graph-Aware Calibration is Key to Building Trustworthy GNNs Under Attack}

In this section, we examine how increasing the perturbation budget affects calibration robustness. While ECE quantifies the absolute level of miscalibration after an attack, the increase of ECE captures the relative degradation from a model’s calibrated state. Both metrics allow us to better assess both the severity of miscalibration and the robustness of each calibration method under attack. In real-world scenarios, stealthy attacks should ideally introduce only a limited number of edge modifications per graph. To simulate this, we vary the attack budget from 0.5$\times$ to 2.0$\times$ the average node degree of each dataset. Figure \ref{fig:budget_chart} reveals significant variability in robustness among different calibration methods when exposed to increasing adversarial budgets. TS, MS, and VS exhibit substantial degradation in calibration quality, with TS showing the greatest vulnerability. Conversely, methods leveraging graph structure, specifically DCGC, demonstrate stronger resilience and stability, indicated by a plateau in ECE after initial increases. Notably, absolute ECE values alone can be misleading; the relative increase in ECE from the initial point (e.g., 0.5) provides a clearer measure of vulnerability. Thus, graph-aware and data-centric calibration methods, such as DCGC, are recommended to enhance robustness, as they consistently maintain calibration quality under adversarial conditions.

\subsection{Designing Trustworthy GNNs: Lessons from Graph-Aware and Data-Centric Calibration}

To explore how the size of the calibration set influences robustness under attack, we conducted experiments varying the calibration set size from 0.25$\times$ to 2$\times$ the size of the test set. This setting allows us to evaluate the scalability and stability of different calibration methods as more calibration data becomes available.

Figure \ref{fig:calib_size} shows that ECE of models calibrated with classical methods like TS and VS tends to plateau or degrade with larger calibration sets, similar to what is shown in Theorem \ref{generalization_theorem}, in which ECE increases when the model generalizes with more data. In contrast, MS consistently maintains performance across calibration sizes. The graph-aware approach, CaGCN, performs exceptionally well with limited calibration data but shows deterioration with larger datasets. Conversely, the data-centric method DCGC steadily improves calibration accuracy as the calibration set expands, demonstrating superior stability and scalability. These results suggest that graph-aware calibration methods such as CaGCN are preferable in low-data regimes, whereas data-centric methods such as DCGC are more suitable when abundant calibration data are available, offering a more robust and reliable option under adversarial conditions and further highlighting the importance of data-centric robustness~\cite{cuong2024curious}.
\section{Conclusion}

We present the first dedicated framework for calibration attacks on GNNs. Our theoretical analysis and extensive empirical evaluation across multiple benchmarks and calibration strategies demonstrate that GNNs, even when well-calibrated, can be adversarially manipulated to produce highly unreliable confidence scores without changing predicted labels. This highlights a critical insight for the community that improving predictive performance alone is insufficient for trustworthiness. Moving forward, it is essential to develop graph-aware and data-centric calibration methods that remain robust in adversarial environments, ensuring the reliability of GNNs in real-world, safety-critical applications.

\section*{Acknowledgments}

Wang is supported by the Army Research Office (ARO) under grant number W911NF-2110198 and Cisco Faculty Research Award. Cheng is supported by the National Science Foundation (NSF) Grant \#2312862, NSF-Simons SkAI Institute, NSF CAREER \#2440542, NSF \#2533996, National Institutes of Health (NIH) \#R01AG091762, NSF ACCESS Computing Resources, a Google Research Scholar Award, and Cisco gift grant. Dung and Hieu are supported by Grant VUNI.2526.CAIR.03, Center for AI Research, VinUniversity. 

\newpage 
\clearpage

\bibliographystyle{ACM-Reference-Format}
\bibliography{ref}

\appendix
\clearpage
\newpage

\twocolumn[
\begin{center}
    {\LARGE \bf Appendix}
\end{center}
]

\section{Details of the Attack Algorithm}
\label{algo_supp}

\vspace{-0.15in}
\begin{algorithm}
\caption{Naive Graph Calibration Attack}
\begin{algorithmic}[1]
\Require Victim Model $f$, Attack Budget $\Delta$, Input Adjacency Matrix $\mathbf{A}_{0}$, Attacked Node Index $u$, Attack Strategy $\mathcal{T}$.

\Ensure Adversarial Adjacency Matrix $\mathbf{A'}$

\State $\mathbf{A'} \gets \mathbf{A}_0$
\For{$t = 1$ \textbf{to} $\Delta$}
    \State $\mathcal{L} \gets$ \Call{DetermineLoss}{$\mathbf{A'}$}
    \State $v \gets$ \Call{TopEdge}{$\mathcal{L}$, $\mathbf{A'}$} \Comment{Select $B$ potential edges}
    \State $\mathbf{A'}_{uv}\gets\mathbf{A'}_{uv}+\mathbf{M}_{uv}$ \Comment{Update adversarial graph}
\EndFor

\State \textbf{return} $\mathbf{A'}$

\State

\Procedure{DetermineLoss}{$\mathbf{A}$}
\State $\hat{y} \gets \mathop{\arg\max}_{k} f\left(\textbf{A},\mathbf{X}\right)_{u,k}$
\State $\hat{p} \gets \text{Softmax}\left(f\left(\mathbf{A},\mathbf{X}\right)_{u}\right)$
\If{$\mathcal{T}$ \textbf{is} \texttt{OCA}}
    \State \textbf{return} $1-\hat{p}_{\hat{y}}$
\ElsIf{$\mathcal{T}$ \textbf{is} \texttt{UCA}}
    \State $j \gets \mathop{\arg\max}_{k,k\neq \hat{y}} f\left(\textbf{A},\mathbf{X}\right)_{u,k}$
    \State \textbf{return} $\hat{p}_{\hat{y}}-\hat{p}_{j}$
\EndIf
\EndProcedure

\State

\Procedure{TopEdge}{$\mathcal{L}$, $\mathbf{A}$}
\State $\Delta\mathcal{L} \gets \frac{\partial \mathcal{L}}{\partial \mathbf{A}}$
\State $S\gets S_{u}$ \Comment{Just select edges connecting to node $u$}
\State \textbf{return} $\arg\max S$
\EndProcedure

\end{algorithmic}
\label{alg:naive}
\end{algorithm}
\vspace{-0.15in}

In this section, we compare the novelty of our proposed algorithm in~\ref{alg:compare} with the naive framework for adversarial attack for graph calibration in Algorithm~\ref{alg:naive}. The novelty of our proposed algorithm is clear. The prominent difference generally lie in the usage of beam search in Lines 2, 4, 8, 11, colored in green in Algorithm \ref{alg:compare}. The second difference is the function \textsc{DetermineLoss}. It has been added for the case when the prediction is flipped, in Lines 21, 22 in Algorithm \ref{alg:compare}, and the loss function for UCA has been modified, in Lines 26, 27 in Algorithm \ref{alg:compare}. The third difference is in the function \textsc{TopEdge}, where the reranking term is added in line 31 in Algorithm \ref{alg:compare}. 

\section{Complexity \& Runtime Analysis}
\label{complexity}

In this section, we analyze the time/memory complexity for Algorithm~\ref{alg:naive} and Algorithm~\ref{alg:compare}.

\textbf{Notation.} We denote the attack budget (the number of perturbation steps) by $\Delta$, the beam width (top-B states for expanding each step) by $B$, and the size of the candidate edge set around the attacked node $u$ by $d_u$. Let $T_f(A)$ denote the cost of one forward pass to obtain the logits/probabilities of node $u$ under adjacency matrix $A$, including the GNN and softmax computation, and let $T_b(A)$ denote the cost of one backward pass to compute the adjacency gradient $\partial L / \partial A$. Let $|Q|$ denote the size of the priority queue.

\textbf{Algorithm~\ref{alg:naive} (Naive Graph Calibration Attack).} Algorithm~\ref{alg:naive} advances one state per iteration: it computes the loss and gradient a single time, selects one edge, and applies one update. Specifically,
\begin{itemize}
    \item forward/backward once: \textcolor{red}{$T_f + T_b$}
    \item select best edge among $d_u$ candidates: \textcolor{red}{$O(d_u)$} 
\end{itemize}

So: $T_{\text{naive}} = O\!\left(\Delta\cdot\left(T_b+T_f + O(d_u)\right)\right)\approx O\!\left(\Delta\cdot\left(T_b+T_f + d_u\right)\right).$

\textbf{Algorithm 2 (UGCA):} At each iteration ($t = 1\dots\Delta$), the algorithm proceeds as follows:

\begin{itemize}
    \item select \textsc{Top-B} states from beam $Q$: \textcolor{red}{$B$}
    \item For each selected state $A$, perform $\Delta$ perturbation steps:
    \begin{itemize}
        \item Forward/backward once (\textcolor{red}{$T_f + T_b$}), select best edge among $d_u$ candidates (\textcolor{red}{$O(d_u)$}), and push into beam (\textcolor{red}{$\log\!\left(\left|Q\right|\right)$}).
    \end{itemize}
\end{itemize}

So: $ T_{\text{UGCA}} \approx O\!\left(B\cdot\Delta\cdot\left(T_b + T_f + d_u + \log\!\left(\left|Q\right|\right)\right)\right).$

We cap the beam with $B$ and number of perturbatiton $\Delta$ at 5, which bounds the search space $\left|Q\right|$ to 3125. For large-scale graphs such as OGBN-Arxiv and Reddit, the overhead terms introduced by beam search, namely $\Delta,\,B,\, \log\!\left(\left|Q\right|\right)$ are negligible compared to the dominant costs of forward and backward passes $T_b,\,T_f$ and neighborhood exploration $d_u$. As a result, the overall runtime of the naive method and UGCA are effectively the same, i.e., $T_{\text{naive}} \approx T_{\text{UGCA}}$.

\section{Extension to other graph-based tasks}\label{extension_task}
We clarify that our attack framework is not limited to node classification. The key requirement is that the model produces class probability outputs and is differentiable with respect to the adjacency matrix. Therefore, our method can be naturally extended to graph classification tasks.

In this paper, we focus on edge perturbations. We argue this setting is practically motivated. For example, in a banking fraud detection system, an attacker might simulate a set of fake transactions (edges), but manipulating node features such as user age is less feasible. 
To instantiate a feature matrix perturbation attack, we can consider the continuous nature of features, and apply projected gradient descent (PGD) or penalty methods directly. 

For instance, one could enforce a norm-bound constraint $ \left| x^{\mathrm{adv}}_{\text{target}} - x_{\text{target}} \right|_p \leq \delta,$ and perform direct gradient-based optimization on node features.

\section{Missing Proofs in Section~\ref{analyze}}
\label{robustness_proof}

\subsection{Proof of Theorem~\ref{generalization_theorem}}\label{sec:proof_gen}

It is evident that $N$ predictions are categorized into $\mathcal{M}$ bins. We first denote $\mathcal{B}_{m}^{k}$ is a set of predictions of class $k$ falls into bin $m$.
\[
\textrm{Class-wise } \mathrm{ECE} = \frac{1}{K}\sum_{k=1}^{K}\sum_{m=1}^{\mathcal{M}}\frac{\left|\mathcal{B}_{m}^{k}\right|}{N}\left|\textrm{acc}\left(\mathcal{B}_{m}^{k}\right)-\textrm{conf}\left(\mathcal{B}_{m}^{k}\right)\right|,
\]

where for each bin $m$ and class $k$:
\begin{itemize}
    \item $\textrm{acc}\left(\mathcal{B}_m^k\right) = \frac{1}{\left|\mathcal{B}_m^k\right|} \sum_{i \in B_m^k} \mathbf{1}\left(y_i = k\right)$: fraction of true labels equal to class $k$
    \item $\text{conf}\left(\mathcal{B}_m^k\right) = \frac{1}{\left|\mathcal{B}_m^k\right|} \sum_{i \in \mathcal{B}_m^k} p_i^{(k)}$: average predicted probability for class $k$
\end{itemize}

Denote $\mathrm{ECE}_{\textrm{OCA}}$ the ECE of the worst case in the overconfidence attack. In the worst case of an overconfidence attack, the confidence of the attacked (predicted) class is 1, while the remaining classes are 0. Hence, there are only 2 bins, $\mathcal{B}_1^k$ and $\mathcal{B}_\mathcal{M}^k$. Under this setup, we aim to prove that
\[
\blacksquare\quad\textrm{Class-wise } \mathrm{ECE}_{\textrm{OCA}} = \frac{1}{K} \left[(K-2)\textrm{acc}(\mathcal{D})+1\right]
\]
\begin{align}
\mathrm{ECE}_{\textrm{OCA}_{k}}
  &= \frac{\left|\mathcal{B}_{1}^k\right|}{N}
     \left|\textrm{acc}\!\left(\mathcal{B}_{1}^k\right)
          -\textrm{conf}\!\left(\mathcal{B}_{1}^k\right)\right|
     \notag \\
  &\quad + \frac{\left|\mathcal{B}_{\mathcal{M}}^k\right|}{N}
     \left|\textrm{acc}\!\left(\mathcal{B}_{\mathcal{M}}^k\right)
          -\textrm{conf}\!\left(\mathcal{B}_{\mathcal{M}}^k\right)\right|.
\end{align}

The confidence of the first and last bins, $\textrm{conf}\left(\mathcal{B}_{1}^k\right)$ and $\textrm{conf}\left(\mathcal{B}_{\mathcal{M}}^k\right)$ euqals $0$ and $1$. The accuracy of each bin $\mathcal{B}_{m}^k$ lies in $[0,1]$. Hence, 
\begin{align}
\mathrm{ECE}_{\textrm{OCA}_{k}}
  &= \frac{\left|\mathcal{B}_{1}^k\right|}{N}
     \left|\textrm{acc}\!\left(\mathcal{B}_{1}^k\right)-0\right|
     + \frac{\left|\mathcal{B}_{\mathcal{M}}^k\right|}{N}
       \left|\textrm{acc}\!\left(\mathcal{B}_{\mathcal{M}}^k\right)-1\right| \notag \\
  &= \frac{\left|\mathcal{B}_{1}^k\right|}{N}
      \,\textrm{acc}\!\left(\mathcal{B}_{1}^k\right)
     + \frac{\left|\mathcal{B}_{\mathcal{M}}^k\right|}{N}
       \Bigl(1-\textrm{acc}\!\left(\mathcal{B}_{\mathcal{M}}^k\right)\Bigr).
\end{align}

Let $b_{ij}$ denote the number of predictions where the predicted and groundtruth classes are $i$ and $j$. Therefore, the total number of predictions is $N=\sum_{i = 1}^{K}\sum_{j = 1}^{K} b_{ij}$. For a given class $k$, the first bin $\mathcal{B}_1^k$ consists of predictions where $k$ is not the predicted class. This corresponds to $\sum_{i = 1, i\neq k}^{K}\sum_{j = 1}^{K} b_{ij}$ predictions. Conversely, the last bin $\mathcal{B}_{\mathcal{M}}^k$ includes all predictions where $k$ is the predicted labels, including $\sum_{j=1}^{K}b_{kj}$ predictions.

In the first and last bin $\mathcal{B}_1^k$ and $\mathcal{B}_{\mathcal{M}}^k$, there are $\sum_{i=1,i\neq k}^{K}b_{ii}$ and $b_{kk}$ true predictions. Hence $\textrm{acc}\left(\mathcal{B}_1^k\right)=\frac{\sum_{i=1,i\neq k}^{K}b_{ii}}{\sum_{i = 1, i\neq k}^{K}\sum_{j = 1}^{K} b_{ij}}$ and $\textrm{acc}\left(\mathcal{B}_\mathcal{M}^k\right)=\frac{b_{kk}}{\sum_{j=1}^{K}b_{kj}}$. Thus, we have the following:
\begin{align}
\mathrm{ECE}_{\textrm{OCA}_{k}}
  &= \frac{\sum_{\substack{i = 1 \\ i\neq k}}^{K}\sum_{j = 1}^{K} b_{ij}}
          {\sum_{i=1}^{K}\sum_{j = 1}^{K}b_{ij}}
     \,\textrm{acc}\!\left(\mathcal{B}_{1}^k\right) \notag \\
  &\quad + 
     \frac{\sum_{j = 1}^{K} b_{kj}}
          {\sum_{i=1}^{K}\sum_{j = 1}^{K}b_{ij}}
     \left(1-\textrm{acc}\!\left(\mathcal{B}_{\mathcal{M}}^k\right)\right) \\
  &= \frac{\sum_{\substack{i=1 \\ i\neq k}}^{K}\sum_{j=1}^{K} b_{ij}}
           {\sum_{i=1}^{K}\sum_{j = 1}^{K}b_{ij}}
     \cdot 
     \frac{\sum_{\substack{i=1 \\ i\neq k}}^{K} b_{ii}}
          {\sum_{\substack{i=1 \\ i\neq k}}^{K}\sum_{j = 1}^{K} b_{ij}} 
     \notag \\
  &\quad + 
     \frac{\sum_{j=1}^{K} b_{kj}}
          {\sum_{i=1}^{K}\sum_{j = 1}^{K} b_{ij}}
     \left(1-\frac{b_{kk}}{\sum_{j=1}^{K} b_{kj}}\right) \\
  &= \frac{\sum_{\substack{i=1 \\ i\neq k}}^{K} b_{ii}}
          {\sum_{i=1}^{K}\sum_{j = 1}^{K} b_{ij}}
     + \frac{\sum_{j=1}^{K} b_{kj}}
            {\sum_{i=1}^{K}\sum_{j = 1}^{K} b_{ij}}
     - \frac{b_{kk}}
            {\sum_{i=1}^{K}\sum_{j = 1}^{K} b_{ij}} \\
  &= \frac{\sum_{\substack{i=1 \\ i\neq k}}^{K} b_{ii} 
          + \sum_{\substack{j=1 \\ j \neq k}}^{K} b_{kj}}
          {\sum_{i=1}^{K}\sum_{j = 1}^{K} b_{ij}}.
\end{align}

We can also compute the class-wise ECE as follows:
\begin{align*}
    & \textrm{Class-wise } \mathrm{ECE}_\textrm{OCA} \\ 
    = & \frac{1}{K}\sum_{k=1}^{K} \mathrm{ECE}_{\textrm{OCA}_{k}}  \\
    =& \frac{1}{K}\sum_{k=1}^{K} \frac{\sum_{i=1,i\neq k}^{K}b_{ii} + \sum_{j=1, j \neq k}^{K}b_{kj}}{\sum_{i = 1}^{K}\sum_{j = 1}^{K} b_{ij}} \\ 
    = & \frac{1}{K} \frac{\sum_{k=1}^{K}\sum_{i=1,i\neq k}^{K}b_{ii} + \sum_{k=1}^{K}\sum_{j=1, j \neq k}^{K}b_{kj}}{\sum_{i = 1}^{K}\sum_{j = 1}^{K} b_{ij}} \\ 
    = & \frac{1}{K}\frac{\left(K-1\right)\sum_{i = 1}^{K}b_{ii} + \left(\sum_{k=1}^{K}\sum_{j=1}^{K}b_{kj} - \sum_{k=1}^K b_{kk}\right)}{\sum_{i = 1}^{K}\sum_{j = 1}^{K} b_{ij}}\\
    = & \frac{1}{K}\frac{\left(K-1\right)\sum_{i = 1}^{K}b_{ii} + \left(\sum_{k=1}^{K}\sum_{j=1}^{K}b_{kj} - \sum_{i=1}^{K}b_{ii}\right)}{\sum_{i = 1}^{K}\sum_{j = 1}^{K} b_{ij}} \\
    = & \frac{1}{K}\frac{\left(K-2\right)\sum_{i = 1}^{K}b_{ii} + \sum_{i=1}^{K}\sum_{j=1}^{K}b_{ij}}{\sum_{i = 1}^{K}\sum_{j = 1}^{K} b_{ij}}.
\end{align*}

Since $\textrm{acc}\left(\mathcal{D}\right) = \frac{\sum_{i=1}^{K}b_{ii}}{\sum_{i = 1}^{K}\sum_{j = 1}^{K} b_{ij}}$, it follows that

\[
\textrm{Class-wise } \mathrm{ECE}_\textrm{OCA} = \frac{1}{K} [\left(K-2\right)\textrm{acc}\left(\mathcal{D}\right)+1].
\]

In underconfidence attack, the attacker tries to degrade the confidence of the predicted class to $\frac{1}{K}$, also leading the confidence of other classes to $\frac{1}{K}$. In this setting, we want to prove
\[
\blacksquare\quad \textrm{Class-wise }\mathrm{ECE}_{\textrm{UCA}} = \textrm{acc}\left(\mathcal{D}\right)-\frac{1}{K}.
\]

We assume that the confidences of all classes of predictions equal $\frac{1}{K}$ and are in the same bin $\mathcal{B}_{\left\lceil\frac{\mathcal{M}}{K}\right\rceil}$. Denote $\mathrm{ECE}_{\textrm{UCA}_{k}}$ the ECE of class $k$ in the worst case of an underconfidence attack. Hence, $N=\left|\mathcal{B}_{\left\lceil\frac{\mathcal{M}}{K}\right\rceil}\right|$ and $\textrm{acc}\left(\mathcal{B}_{\left\lceil\frac{\mathcal{M}}{K}\right\rceil}\right)=\textrm{acc}\left(\mathcal{D}\right)$. Consequently,
\[
\mathrm{ECE}_{\textrm{UCA}_{k}} = \frac{\left|\mathcal{B}_{\left\lceil\frac{\mathcal{M}}{K}\right\rceil}\right|}{N}\left|\textrm{acc}(\mathcal{B}_{\left\lceil\frac{\mathcal{M}}{K}\right\rceil})-\textrm{conf}(\mathcal{B}_{\left\lceil\frac{\mathcal{M}}{K}\right\rceil})\right|
\]

\[
= \left|{\textrm{acc}\left(\mathcal{D}\right)-\frac{1}{K}}\right| = \textrm{acc}\left(\mathcal{D}\right)-\frac{1}{K}.
\]

\[
\textrm{Class-wise } \mathrm{ECE}_{\textrm{UCA}} = \frac{1}{K}\sum_{k=1}^{K}\mathrm{ECE}_{\textrm{UCA}_k} = \textrm{acc}\left(\mathcal{D}\right)-\frac{1}{K}.
\]

Thus, the proof is finished. 

\subsection{Proof of Corollary~\ref{generalization_corollary}}\label{sec:proof_cor}

In Corollary \ref{generalization_corollary}, we analyze two GNNs with identical classification accuracy but deployed on tasks of differing complexity. Specifically, we assume that the first task involves more classes than the second, i.e., $K_1 > K_2$. Consider the function
$
y = \frac{1}{x} \left[(x - 2)\alpha + 1\right] = \frac{\alpha x - 2\alpha + 1}{x},
$ 
where $\alpha > \frac{1}{2}$. Taking the derivative w.r.t. $x$, we obtain
$$
\frac{dy}{dx} = \frac{2\alpha - 1}{x^2} > 0,
$$
since $2\alpha - 1 > 0$. This implies that $y$ is a monotonically increasing function with respect to $x$, and thus $\mathrm{ECE}_{\textrm{OCA}}$ increases with the number of classes $K$. Additionally, it is straightforward to observe that $\mathrm{ECE}_{\textrm{UCA}}$ also increases with $K$. Therefore, when $K_1 > K_2$, it follows that
$$
\mathrm{ECE}_{\textrm{OCA}_1} > \mathrm{ECE}_{\textrm{OCA}_2} \quad \text{and} \quad \mathrm{ECE}_{\textrm{UCA}_1} > \mathrm{ECE}_{\textrm{UCA}_2}.
$$

Thus, we finish the proof.

\end{document}